\documentclass{article}

\usepackage{microtype}
\usepackage{graphicx}
\usepackage{subfigure}
\usepackage{booktabs} % for professional tables

\usepackage{graphicx}
\usepackage{textcomp}
\usepackage{subfigure}
\usepackage{makecell}
\usepackage{balance}
\usepackage{bm}
\usepackage{amsfonts}
\usepackage{multirow}
\usepackage{algorithm}
\usepackage{algorithmic}
\usepackage{amsmath}
\usepackage{amsthm}
\usepackage{amssymb,color}
\usepackage{float}
\usepackage{graphicx}
\usepackage{caption}
\usepackage{ragged2e}

\newtheorem{definition}{Definition}

\newtheorem{theorem}{Theorem}

\usepackage{hyperref}

\usepackage[accepted]{icml2021}

\begin{document}

\twocolumn[
\icmltitle{Towards Speeding up Adversarial Training in Latent Spaces}



\icmlsetsymbol{equal}{*}

\begin{icmlauthorlist}
\icmlauthor{Yaguan Qian}{equal,zust}
\icmlauthor{Qiqi Shao}{equal,zust}
\icmlauthor{Tengteng Yao}{zust}
\icmlauthor{Bin Wang}{hik}
\icmlauthor{Shouling Ji}{zju}
\icmlauthor{Shaoning Zeng}{hzu}
\icmlauthor{Zhaoquan Gu}{gzu}
\icmlauthor{Wassim Swaileh}{cy}
\end{icmlauthorlist}

\icmlaffiliation{zust}{School of Big-Data Science, Zhejiang
University of Science and Technology, Hangzhou, China}
\icmlaffiliation{hik}{Network and Information Security Laboratory of
Hangzhou Hikvision Digital Technology Co., Ltd., China}
\icmlaffiliation{zju}{College of Computer Science and Technology, Zhejiang University, Hangzhou, China}
\icmlaffiliation{hzu}{School of Computer Science and Technology, Huizhou University, Huizhou, China}
\icmlaffiliation{gzu}{Cyberspace Institute of Advanced Technology, Guangzhou
University, Guangzhou, China}
\icmlaffiliation{cy}{ETIS Research Laboratory, CY Cergy Paris University, Paris, France}

\icmlcorrespondingauthor{Bin Wang}{wbin2006@gmail.com}
\icmlcorrespondingauthor{Yaguan Qian}{qianyaguan@zust.edu.cn}

% \icmlkeywords{Machine Learning, ICML}  %%%%%%%%%%%%%%%%%%%%%%%%%

\vskip 0.3in
]




% \printAffiliationsAndNotice{\icmlEqualContribution}  %%%%%%%%%%%

% otherwise use the standard text.

\begin{abstract}
Adversarial training is wildly considered as one of the most effective way to defend against adversarial examples. However, existing adversarial training methods consume unbearable time, due to the fact that they need to generate adversarial examples in the large input space. To speed up adversarial training, we propose a novel adversarial training method that does not need to generate real adversarial examples. By adding perturbations to logits to generate Endogenous Adversarial Examples (EAEs)---the adversarial examples in the latent space, the time consuming gradient calculation can be avoided.  Extensive experiments are conducted on CIFAR-10 and ImageNet, and the results show that comparing to state-of-the-art methods, our EAE adversarial training not only shortens the training time, but also enhances the robustness of the model and has less impact on the accuracy of clean examples than the existing methods. 
\end{abstract}

\section{Introduction}
\label{submission}

Deep neural networks (DNNs) has been successfully applied in image recognition \cite{[a21]}, speech recognition \cite{[a22]}, natural language processing \cite{[a23]}, and other fields. However, recent studies revealed that DNNs are seriously vulnerable to adversarial examples. Many methods were proposed to generate adversarial examples \cite{[a1]}, \cite{[a2]}, \cite{[a3]}, \cite{[a4]}, \cite{[a5]}, \cite{[a6]}. At the same time, the corresponding defenses have also been extensively studied \cite{[a2]}, \cite{[a5]}, \cite{[a7]}, \cite{[a11]}, \cite{[a12]}, in which adversarial training is considered to be one of the most effective defensive methods at present \cite{[a12]}.

\begin{figure}%%图
\centering  %插入的图片居中表示
\includegraphics[width=1in, height=1in]{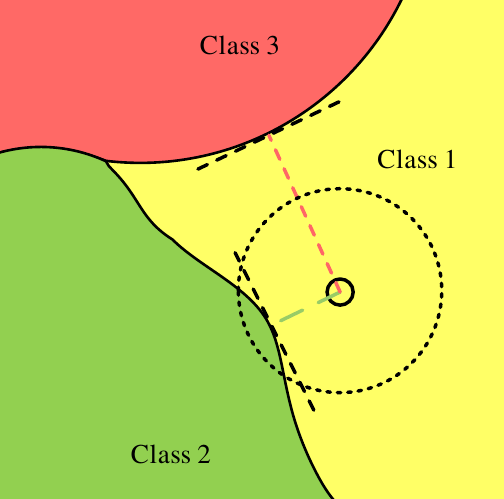}  
\includegraphics[width=1in, height=1in]{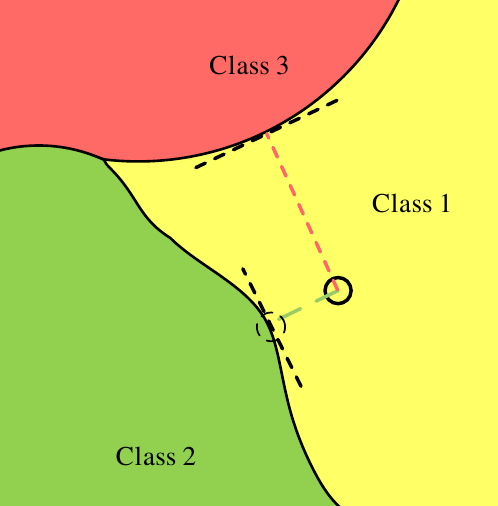}
\includegraphics[width=1in, height=1in]{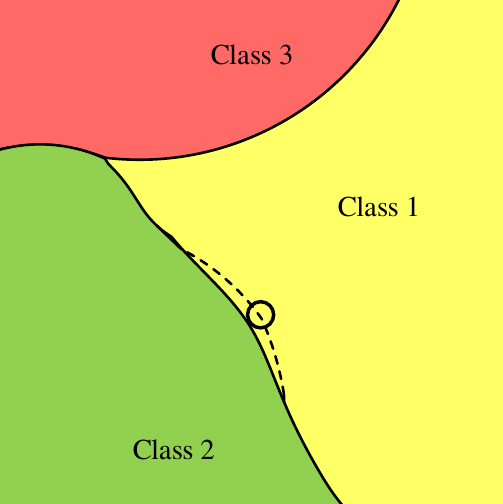} %插入的图，包括JPG,PNG,PDF,EPS等，放在源文件目录下
\caption{Illustration of the decision boundary between Class1, Class2, and Class3. The circle with solid line represents the example $\boldsymbol{x}$, and the circle with dotted line represents the perturbation bound. Left: distance to the decision boundary. Middle: an adversarial example in the input space corresponding to an EAE. Right: new decision boundary after EAE adversarial training.}  %图片的名称
 \label{fig14}   %标签，用作引用
 \end{figure}

Existing research shows that the computational cost of adversarial training is very high \cite{[a12]}. In particular, the generation of adversarial examples accounts for the main part of the total time-consumption. For improving the efficiency of adversarial training, Shafahi et. al \cite{[a12]} proposed a ``Free” method to reduce the number of gradient steps used for generating adversarial examples and Wong et. al \cite{[a13]} proposed a ``Fast” method recently that outperforms the ``Free” method. However, these methods still rely on creating real adversarial examples in the input space. 

In this paper, we propose a novel adversarial training method without generating real adversarial examples in the input space. The main idea is to add perturbations to the output (logits) of the penultimate layer. This output space is termed as a \emph{latent space} in this paper. In contrast to an adversarial example in the input space, we denote an example added perturbations in the latent space as an \emph{endogenous adversarial example} (EAE).  Suppose $\mathcal{F}(\boldsymbol{x})$ is a pretrained model (or a classifier), and its logits $\boldsymbol{z}=(z_1, z_2, \dots, z_{C})$ of $\boldsymbol{x}$, where $C$ is the number of classes. Assuming $z_1$ is the largest component and $z_2$ is the second largest component, then the class (\emph{Class1}) corresponding to $z_1$ is considered as the predicted class with the maximum posterior rule. Accordingly, we denote the class responding to $z_2$ as \emph{Class2}. Due to the similarity of the confidence of Class1 and Class2, we speculate the example $\boldsymbol{x}$ is closer to the decision boundary of Class2 than that of the remained classes in the input space as illustrated in Fig. \ref{fig14}. Thus we think Class2 is suitable for an adversarial targeted class in the latent space to ensure the smallest perturbation. 

Unfortunately, even if we get the smallest perturbation in the latent space, we still do not guarantee its corresponding perturbation in the input space within the range of constraints, which is critical to avoid detection by human. We define the concepts of \emph{seed examples} that can be successfully crafted as an adversarial example by a specific attack method. We notice the logit difference between Class1 and Class2 of a seed example and a non-example has significant difference in statistical distributions. Thus we can determine the seed examples via their logit difference, which guarantees their perturbations in the latent space satisfying the constraint in the input space.

The advantage of this approach is that it can avoid calculating the gradient of a loss with respect to the example and remarkably speed up the adversarial training process. Compared to previous work, the contributions of our work are summarized as follows:
\begin{enumerate}
  \item [(1)] To the best of our knowledge, we are the first to propose EAE adversarial training without generating real adversarial examples. Since no need to calculate the gradient of the loss with respect to the example, it remarkably improves the adversarial training speed. 
  \item [(2)] We clearly define the concepts of seed examples and EAEs for the first time. With these concepts, we observe the difference of seed example and non-seed example in statistics, which helps craft an EAE satisfying the perturbation constraints. 

  \item [(3)] Extensive experiments conducted on CIFAR-10 and ImageNet show that the training time required for EAE adversarial training is about 1/3 of ``Fast" adversarial training. Meanwhile, for clean examples, the model after EAE adversarial training does not significantly reduce its classification accuracy.
   
\end{enumerate}

\section{Related Work}

Szegedy et al. \cite{[a1]} first discovered the existence of adversarial examples in DNNs. After then, Goodfellow et al. \cite{[a2]} proposed FGSM for quickly generating adversarial examples through one-step update along the gradient direction. Subsequently, Kurakin et al. \cite{[a14]} proposed BIM, which is a multi-step attack with smaller steps to achieve a stronger adversarial example. Madry et al. \cite{[a5]} proposed PGD, strengthening iterative adversarial attacks by adding multiple random restart steps.  All these methods craft adversarial examples in the input space, while our method manipulate logits in the latent space.

Goodfellow et al. \cite{[a2]} proposed FGSM adversarial training method using FGSM adversarial examples. Madry proposed PGD adversarial training, which is considered as the best method so far. However, existing research has shown \cite{[a12]} that the computational complexity of adversarial training is very high, and the generation of adversarial examples accounts for the main part of the total time. To reduce the computational overhead of PGD, Shafahi et al. \cite{[a12]} proposed ``Free" adversarial training, which updated both model weights and input perturbation by using a single back-propagation. Zhang et al. \cite{[a18]} believed that when performing a multi-step PGD, redundant calculations can be cut down. Therefore, Wong et al. \cite{[a13]} proposed the ``Fast" method, which used the previous minbatch’s perturbation or initialized a perturbation from uniformly random perturbation to add to the clean example. However, these methods all rely on generating real adversarial examples, and calculate the gradient of loss with respect to the example through a back-propagation. 

%\IEEEpubidadjcol

\section{Problem Statement}
\subsection{Basic Definition}
In this paper, we denote training data as ${D}=\left\{\left(\boldsymbol{x}_{i}, {y}_{i}\right)\right\}_{i=1}^{N}$, where $\boldsymbol{x}_{i} \in\mathbb{R}^{D}$, $y_{i} \in\{1,\dots,C\}$, and $y_{i}$ is the class index of $\boldsymbol{x}_{i}$. Suppose $\mathcal{F}(\boldsymbol{x};\boldsymbol{\theta})$ is a pretrained neural network model with a parameter vector $\boldsymbol{\theta}$, then $\mathcal{F}(\boldsymbol{x};\boldsymbol{\theta})=\mathcal{F}^{(N)}\left(\ldots\left(\mathcal{F}^{(2)}\left(\mathcal{F}^{(1)}(\boldsymbol{x})\right)\right)\right)$. The output of the penultimate layer of the classifier $\mathcal{F}$ is a vector of logits $\mathcal{F}^{(N-1)}=\boldsymbol{z}=(z_1, z_2, \dots, z_{C})$, and the output of the last layer is a probability vector $P=(p_1, p_2, \dots, p_{C})$ obtained by a Softmax function,
%\begin{equation}
%p_{j}=\operatorname{Softmax}(\boldsymbol{z})_{j}=e^{z_{j}} / \sum_{k=1}^{C} e^{z_{k}},j=1, 2, \dots, C
%\end{equation}
where $p_{j}$ is the confidence score to determine the example classified as the $j$-th class. $\hat{y}=\arg \max _{j}\left\{p_{j}\right\}$ is the predicted class index of $\boldsymbol{x}$.

Szegedy et al. \cite{[a1]} first discovered adversarial examples in DNNs. We formally defines adversarial examples as follows:

\begin{definition}(Perturbed Examples and Adversarial Examples)
For a clean example $\boldsymbol{x} \in\mathbb{R}^{D}$, $y$ is the ground-truth class index of $\boldsymbol{x}$. If $\boldsymbol{x} $ is added by a perturbation $\boldsymbol{\delta}$ with $\left \| \boldsymbol{\delta} \right\|_2<\epsilon$ , i.e., $\boldsymbol{x}^{\prime}=\boldsymbol{x}+\boldsymbol{\delta}$, then $\boldsymbol{x}^{\prime}$ is referred to as a perturbed example. For a perturbed example $\boldsymbol{x}^{\prime}$, if it satisfies $\mathcal{F}(\boldsymbol{x}^{\prime})\neq y$, then $\boldsymbol{x}^{\prime}$ is called an adversarial example.
\end{definition}

Note here the constraint $\left \| \boldsymbol{\delta} \right\|_2<\epsilon$ is critical to avoid detection by human. Since the added perturbation $\boldsymbol{\delta}$ is constrained by $\epsilon$, not all clean examples can be successfully crafted to the corresponding adversarial examples. Accordingly, we further divide the clean examples into \emph{candidate seed examples} and \emph{seed examples} in this paper.

\begin{definition}(Candidate Seed Examples and Seed Examples)
Assuming that $G$ is an algorithm to generate adversarial examples, i.e., $\boldsymbol{x}^{\prime}=G(\boldsymbol{x})$, any clean example $\boldsymbol{x} \in\mathbb{R}^{D}$ that can be accurately classified by $\mathcal{F}$ is called a candidate seed example of $G$. Supposing that $y$ is the ground-truth class of $\boldsymbol{x}$, given a constraint bound $\epsilon$, if $\mathcal{F}(G(\boldsymbol{x}))\neq y$, then $\boldsymbol{x}$ is called a seed example of $G$ on the model $\mathcal{F}$.
\end{definition}

% \begin{definition}(Seed Examples)
% Supposing that $\boldsymbol{x}$ is a candidate seed example, and $y$ is the ground-truth class of $\boldsymbol{x}$, given a constraint bound $\epsilon$, if $\mathcal{F}(G(\boldsymbol{x}))\neq y$, then $\boldsymbol{x}$ is called a seed example of $G$ on the model $\mathcal{F}$.
% \end{definition}

%\subsection{Endogenous Adversarial Examples}
Traditional adversarial training needs to generate adversarial examples in the input space to train the classifier. Instead, our method is to perturb the logits of the penultimate layer of a network, which is called as endogenous adversarial examples. The formal definition is presented as follows:

\begin{definition}(Endogenous Adversarial Examples)
For a clean example $\boldsymbol{x} \in\mathbb{R}^{D}$, $y$ is the ground-truth class of $\boldsymbol{x}$. $\boldsymbol{z}$ is the logits of $\boldsymbol{x}$, and a perturbation $\Delta$ in the latent space is added to the logits $\boldsymbol{z}$, that is, $\boldsymbol{z}^{\prime}=\boldsymbol{z}+\Delta$. If there is a corresponding $\boldsymbol{x}^{\prime}$ in the input space, which satisfies $\mathcal{F}(\boldsymbol{x}^{\prime})=\operatorname{Softmax}(\boldsymbol{z}^{\prime})\neq y$, then $\boldsymbol{z}^{\prime}$ is referenced to as an endogenous adversarial example (EAE).
\end{definition}

\subsection{Problem Setup}
Here we obtain adversarial examples in the latent space of DNNs. Thus we model the adversarial training as a two-level optimization problem as follows:
\begin{equation}
\begin{array}{c}
\min \limits_{\boldsymbol{\theta}} \mathcal{L}(\boldsymbol{z}+\Delta, \boldsymbol{\theta}) \\
\text {s.t. } \underset{\Delta}{\min} \|\Delta\|_{2} \\
z_{y}+\Delta_{y}<\sup\limits_{\Delta_{j}} \left\{z_{j}+\Delta_{j} \mid 1 \leq j \leq C \wedge j \neq y\right\}
\end{array}
\label{equ1}
\end{equation}
where $\mathcal{L}(\cdot)$ is a loss function, and $z_y$ is the component of logits corresponding to the ground-truth class $y$. The up-level problem is to minimize the loss function $\mathcal{L}$, while the low-level problem is to find the minimum perturbation $\Delta$, as described in problem (\ref{equ1}).

Instead of generating adversarial examples in the input space, we attempt to obtain a minimum perturbation $\Delta$ in the latent space instead of perturbation $\boldsymbol{\delta}$ in the input space. Assume that the ground-truth class $y$ is consistent with the predicted class $\hat{y}$ of the clean example $\boldsymbol{x}$, i.e., $y=\hat{y}$. After calculating by the softmax layer, the final output of the EAE $x^{\prime}$ is:
\begin{equation}
\begin{aligned}
p_{j}^{\prime}&=\operatorname{Softmax}(\boldsymbol{z}+\Delta)_{j}\\
&=e^{z_{j}+\Delta_{j}}/\sum\limits_{k=1}^{C}e^{z_{k}+\Delta_{k}}, \qquad j=1,2,\dots,C
\end{aligned}
\end{equation}
where $y^{\prime}=\arg \max _{j}\left\{p_{j}^{\prime}\right\}$ is the predicted class of the EAE. From the definition of adversarial examples, we know that $y^{\prime}\neq y$ and $y^{\prime}\neq\hat{y}$. That is, $y\neq \arg \max_{j}\left\{p_{j}^{\prime}\right\}$. Therefore, $\mathcal{F}(\boldsymbol{x}^{\prime};\boldsymbol{\theta})\neq y$ means $\operatorname{Softmax}(\boldsymbol{z}+\Delta)_{y}<\sup  \left\{\operatorname{Softmax}(\boldsymbol{z}+\Delta)_{j} \mid 1 \leq j \leq C \wedge j \neq y\right\}$. 

Since Softmax is a monotonically increasing function, the above inequality is equivalent to $z_y+\Delta_y<\sup\limits_{\Delta_{j}}  \left\{z_j+\Delta_j \mid 1 \leq j \leq C \wedge j \neq y\right\}$. Therefore, we express the process of generating EAEs as the following optimization problem:
\begin{equation}
\begin{array}{c}
\min \limits_{\Delta}\|\Delta\|_{2} \\\text {s.t.}\quad z_{y}+\Delta_{y}<\sup\limits_{\Delta_{j}}  \left\{z_{j}+\Delta_{j} \mid 1 \leq j \leq C \wedge j \neq y\right\}
\end{array}
\label{equ2}
\end{equation}

\section{Implementation}
Although we can find a minimum perturbation $\Delta$ in the latent space through an optimization algorithm, it cannot guarantee the corresponding perturbation $\boldsymbol{\delta}$ in the input space to meet the constraint $\left \| \boldsymbol{\delta} \right\|_2<\epsilon$, i.e., the EAE does not satisfy the definition of adversarial examples (Definition 1). We address this problem by selecting seed examples. According to the definition of seed examples, they are the clean examples guaranteed to be crafted as adversarial examples. A threshold-based method using statistical distribution is adopted to select seed examples. Then with these selected seed examples, we use an optimization method to obtain EAEs. At the same time, the adversarial training are conducted with these EAEs.

\subsection{Selecting Seed Examples}
 In this paper, we propose a method based on statistical distribution to observe the seed examples. In fact, under the constraint $\left \| \boldsymbol{\delta} \right\|_2<\epsilon$, only these examples nearby the decision boundary can finally be crafted as adversarial examples \cite{[p]}, which are referred to as \textit{seed examples} in this paper. If we select all possible seed examples, then we can craft \textit{endogenous adversarial examples} that $\mathcal{F}(\boldsymbol{x}+\boldsymbol{\delta})=\operatorname{Softmax}(\boldsymbol{z}+\Delta)$, s.t. $\left \| \boldsymbol{\delta} \right\|_2<\epsilon$.

\begin{figure} 
\centering 
\subfigure[CIFAR-10]{
\includegraphics[width=3.8cm,height=3cm]{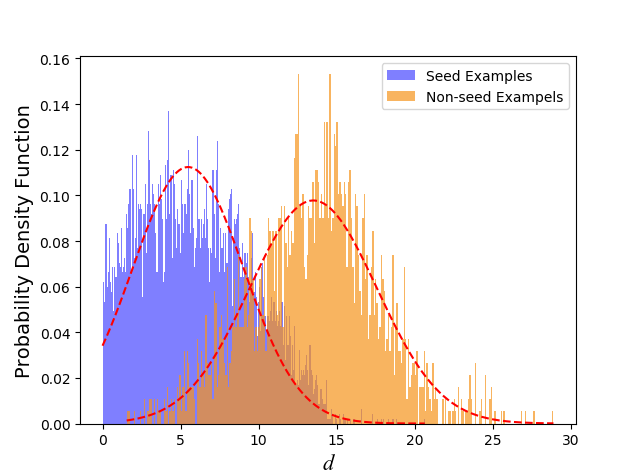} 
}
%  \quad
\subfigure[ImageNet]{
\includegraphics[width=3.8cm,height=3cm]{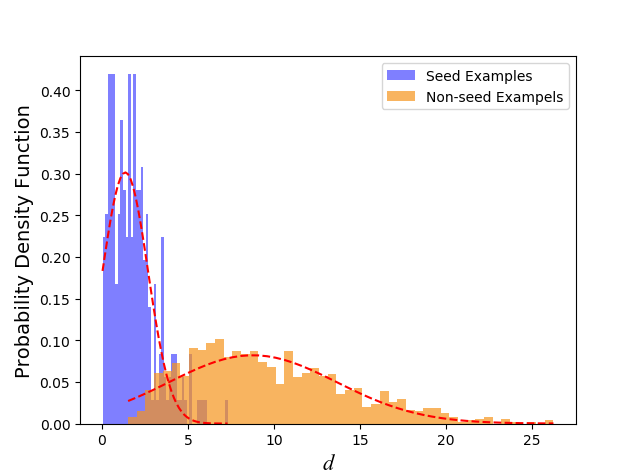} 
} 
 \caption{LD distribution of seed examples and non-seed examples: (a) on CIFAR-10, when FGSM is carried out on $F_1$ and $\epsilon=0.05$, LD distribution of seed examples and non-seed examples; (b) on ImageNet, when FGSM is carried out on $F_3$ and $\epsilon=0.002$, LD distribution of seed examples and non-seed examples.}
 \label{fig.1}
 \end{figure}

Let $d=z_{y}-z_{s}$ denote the difference between the largest component and the second largest component in the logits. For convenient description, we refer to this difference as LD (Logit Difference). Our experiment shows that for seed examples, LD is smaller, while for non-seed examples, LD is larger. Fig. \ref{fig.1} shows that the LDs of both seed examples and non-seed examples follows the Gaussian distribution. Thus, we introduce a threshold $\gamma$ to select the seed examples. If a clean example satisfies $d<\gamma$, we believe it may be a seed example with a large confidence. With this simple rule, we can find most of the seed examples.% to reduce more time for further generating adversarial examples.

We determine the value of $\gamma$ by empirical rules, which we will discuss in the experiment. First, some clean examples that correctly classified by the model are remained as a set of candidate seed examples  ${ \widetilde D}=\left\{\left(\boldsymbol{x}_{i}, {y}_{i}\right)\right\}_{i=1}^{M}$, where $\widetilde D \subset D$ and $|\widetilde D| \ll |D|$. If the clean example $\boldsymbol{x} \in \widetilde D$ can be perturbed successfully by an adversarial generation algorithm $G$, such as FGSM, then $\boldsymbol{x}$ is remained to the set of seed examples $D^+$, otherwise it  is remained to the set of non-seed examples $D^-$. Then, we can obtain the difference between the largest component and the second largest component in the logits of examples in $D^+$ and $D^-$, respectively (as shown in Fig. \ref{fig.1}). We use the mean of LD denoted by $\bar d$, which is called as MLD in this paper, in $D^+$ as the threshold $\gamma$. Since our EAE adversarial training does not strictly require that all seed examples must generate EAEs, the threshold determined by this statistical method is more reliable.

\subsection{Generating EAEs}
Taking a deep insight into problem (\ref{equ2}), we observe that obtaining the minimum perturbation with the constraint only needs to find the second largest value in $z_{j}(1\leqslant j \leqslant y)$  corresponding to the second likelihood class. Let $z_s$ denote the second largest value: $z_{s}=\max \left\{z_j \mid 1 \leq j \leq C \wedge j \neq y\right\}$, such that $z_{y}+\Delta_{y}<z_{s}+\Delta_{s}$, where $\Delta_{y}$ is the perturbation added on $z_y$ and $\Delta_{s}$ is the perturbation added on $z_s$. Obviously, when $\Delta_{j}=0$ $(1\leqslant j \leqslant C \wedge j\neq y \wedge j \neq s)$, $\|\Delta\|_{2}$ reaches the minimum and $\|\Delta\|_{2}=\sqrt{\Delta_{y}^{2}+\Delta_{s}^{2}}$. Since $\min \Delta_{y}^{2}+\Delta_{s}^{2}$ is equivalent to $\min\sqrt{\Delta_{y}^{2}+\Delta_{s}^{2}}$, the optimization problem (\ref{equ2}) can be expressed as:
\begin{equation}
\begin{array}{c}
\underset{\Delta_{y},\Delta_{s}}{\min}\quad \Delta_{y}^{2}+\Delta_{s}^{2}
\\\text {s.t.}\quad z_{y}+\Delta_{y} \leq z_{s}+\Delta_{s} 
\end{array}
\label{equ3}
\end{equation}

Since the constraint $z_{y}+\Delta_{y}<z_{s}+\Delta_{s}$ cannot determine the boundary of feasible regions, we further relax the constraints as $z_{y}+\Delta_{y} \leq z_{s}+\Delta_{s}$ in problem  (\ref{equ3}). Although this relax cannot guarantee that the example added with the perturbation $\Delta$ can become an EAE, we can ensure $z_{y}+\Delta_{y}=z_{s}+\Delta_{s}$, i.e., the confidence score of Class1 and Class2 is equal. Accordingly, the perturbed example has a 50\% probability to be classified as one of the top-2 classes. Our adversarial training does not strictly require each clean example becomes an EAE. 

We use the Lagrangian multiplier method to solve this problem. The Lagrangian function can be written as follows:
\begin{equation}
    L(\Delta_{y},\Delta_{s},\lambda)=\Delta_{y}^{2}+\Delta_{s}^{2}+\lambda(\Delta_{y}-\Delta_{s}+z_{y}-z_{s})
\end{equation}
Its optimality conditions are:
\begin{equation}
\left\{\begin{array}{l}
\frac{\partial L}{\partial \Delta_{y}}=2 \Delta_{y}+\lambda=0 \\
\frac{\partial L}{\partial \Delta_{s}}=2 \Delta_{s}-\lambda=0 \\
\lambda\left(\Delta_{y}-\Delta_{s}+z_{y}-z_{s}\right)=0 \\
\lambda \geq 0
\end{array}\right.
\end{equation}
So when $\lambda=0$, then $\Delta_{y}=0$, $\Delta_{s}=0$; when $\Delta_{y}-\Delta_{s}+z_{y}-z_{s}=0$, then $\Delta_{y}=-(z_{y}-z_{s})/2$, $\Delta_{s}=(z_{y}-z_{s})/2$, $\lambda=z_{y}-z_{s}$. Obviously, the former does not meet the solution of problem (\ref{equ3}). In summary, the approximate optimal solution of problem (\ref{equ2}) is $\Delta^*=(\Delta_{1},\dots,\Delta_{j},\dots,\Delta_{C})$,
\begin{equation}
\Delta^{*}=\left\{\begin{array}{cl}
-\left(z_{y}-z_{s}\right) / 2, & j=y \\
\left(z_{y}-z_{s}\right) / 2, & j=s \\
0, & j \neq y \wedge j \neq s
\end{array}\right.
\end{equation}

\begin{algorithm}
	\caption{Adversarial training with EAEs}%算法标题
     
    {\textbf{Input:} $D=\left\{\left(\boldsymbol{x}_{i}, y_{i}\right)\right\}_{i=1}^{N}$; $T$ is the number of epochs; $K$ is the number of minibatch; $M$ is  the size of minibatch; $\gamma$ is the threshold}\\
    {\textbf{Output:} $\boldsymbol{\theta}$ }

\begin{algorithmic}%一行一个标行号

	\STATE \textbf{for}  $t=1,...,T$ \textbf{do}\\
	\STATE  \quad \textbf{for}  $i=1,...,K$ \textbf{do}\\
	\STATE  \quad\quad \textbf{for}  $j=1,...,M$ \textbf{do}\\ 
	\STATE  \quad\quad \quad $\boldsymbol{z}_{j} \leftarrow F^{(N-1)}\left(\ldots\left(F^{(2)}\left(F^{(1)}\left(\boldsymbol{x}_{j}\right)\right)\right)\right)$
	\STATE  \quad\quad \quad $z_{y} \leftarrow \max(\boldsymbol{z}_{j})$, $z_{s} \leftarrow \sec(\boldsymbol{z}_{j})$
	\STATE \quad\quad \quad \textbf{if}  $z_{y}-z_{s}<\gamma$ \textbf{then}\quad$\boldsymbol{z}_{j} \leftarrow \boldsymbol{z}_{j}+\Delta_{j}$
	\STATE  \quad\quad \quad \textbf{end if}
	\STATE  \quad\quad \textbf{end for}
	\STATE  \quad\quad $\mathcal{L} \leftarrow \frac{1}{M} \sum \limits_{i=1}^{M} \mathcal{L}\left(\boldsymbol{z}_{j}, y_{j}\right)$, $\boldsymbol{\theta} \leftarrow \boldsymbol{\theta}-\eta \cdot \nabla_{\boldsymbol{\theta}} \mathcal{L}$
	\STATE    \quad \textbf{end for}\\
	\STATE  \textbf{end for}\\
		
\end{algorithmic}
\end{algorithm}

\subsection{Adversarial Training with EAEs}
We utilize SGD to train the DNN and divide the training dataset $D$ into $K$ minibatchs. Suppose the initial parameters of the DNN are $\boldsymbol{\theta}_{0}$, the algorithm will generate the EAEs in the first minibatch. With these EAEs, the gradient of the loss function with respect to $\boldsymbol{\theta}$ is obtained and then the model parameters are updated by back-propagation. This iterative update will continue until $K$ minibatchs are processed.

\section{Existence of EAEs}
In this section, we make an attempt to interpret the existence of EAEs through the theory of manifold. We ﬁrst introduce the basic deﬁnitions of the manifold for better understanding of our further analysis.

\begin{definition}(Manifolds)
A $d$-dimensional manifold $M$ is a set that is locally homeomorphic with $\mathbb{R}^d$. That is, for each $\boldsymbol{x}\in M$, there is an open neighborhood $N_{\boldsymbol{x}}$ around $\boldsymbol{x}$, and a homeomorphism $f:N_{\boldsymbol{x}}\to \mathbb{R}^d$. These neighborhoods are referred to as coordinate patches, and the map is referred to as coordinate charts. The image of the coordinate charts is referred to as parameter spaces \cite{[a19]}.
\end{definition}

According to manifold hypothesis, real-world data presented in high-dimensional spaces are expected to concentrate in the vicinity of a manifold $M$ of much lower dimensionality $\mathbb{R}^d$ embedded in $\mathbb{R}^D (D>d)$  \cite{[a16]}. Then, we assume the latent space is a $d$-dimensional manifold $M$ passed through the $D$-dimensional input space.

\begin{theorem}
Assume that $\boldsymbol{x}$ is a seed example and $\boldsymbol{x}'$ is an adversarial example of $\boldsymbol{x}$ in the input space, which means that $\mathcal{F}(\boldsymbol{x'})\neq y$ and $\left \| \boldsymbol{x}-\boldsymbol{x'} \right\|_2<\epsilon$,  then there exist an endogenous adversarial example $\boldsymbol{z'}$ in the latent space and a positive constant $c$ satisfying $\|\boldsymbol{z}'-\boldsymbol{z}\|_2 \leq c\| \boldsymbol{x}' - \boldsymbol{x}\|_2$.
\end{theorem}

\begin{proof}
Let $M$ be a $d$ dimensional manifold, i.e., the latent space. This manifold $M$ is an embedded submanifold of $\mathbb{R}^{D}$. Let $P:\mathbb{R}^{D} \to M$ be the projection operator onto $M$, i.e.,
\begin{equation}
P(\boldsymbol{v}) = \arg\min_{\boldsymbol{w} \in M}\| \boldsymbol{w} -\boldsymbol{v}\|_{2}
\end{equation}

Given a point $\boldsymbol{x}\in \mathbb{R}^{D}$, suppose that there exists a unique projection point $\boldsymbol{x}_*$ of $\boldsymbol{x}$ onto $M$, i.e.,
\begin{equation}
\boldsymbol{x}_*  = P(\boldsymbol{x}).
\end{equation}
By the continuity of projection operator, there exists a positive constant $\epsilon>0$ such that
for any $\boldsymbol{x}'\in B(\boldsymbol{x},\epsilon)$, the projection $P(\boldsymbol{x}')$ is well-defined. Fig. \ref{fig3} illustrates the projection procedure.
Since $B(\boldsymbol{x},\epsilon)$ is compact, there exists a positive constant $c_0>0$ such that
\begin{equation}
\begin{aligned}
{\rm dist}(P(\boldsymbol{x}'),\boldsymbol{x}_*)=& {\rm dist}(P(\boldsymbol{x}'),P(\boldsymbol{x})) \leq c_0\| \boldsymbol{x}' -\boldsymbol{x}\|_2\leq c_0\epsilon,\\
& \boldsymbol{x}'\in B(\boldsymbol{x},\epsilon)
\end{aligned}
\end{equation}
By the above inequality, we have
\begin{equation}
P\big(B(\boldsymbol{x},\epsilon)\big)\subset B(\boldsymbol{x}_*, c_0 \epsilon)
\end{equation}

\begin{figure}%%图
	\centering  %插入的图片居中表示
	\includegraphics[width=5.0cm,height=3.0cm]{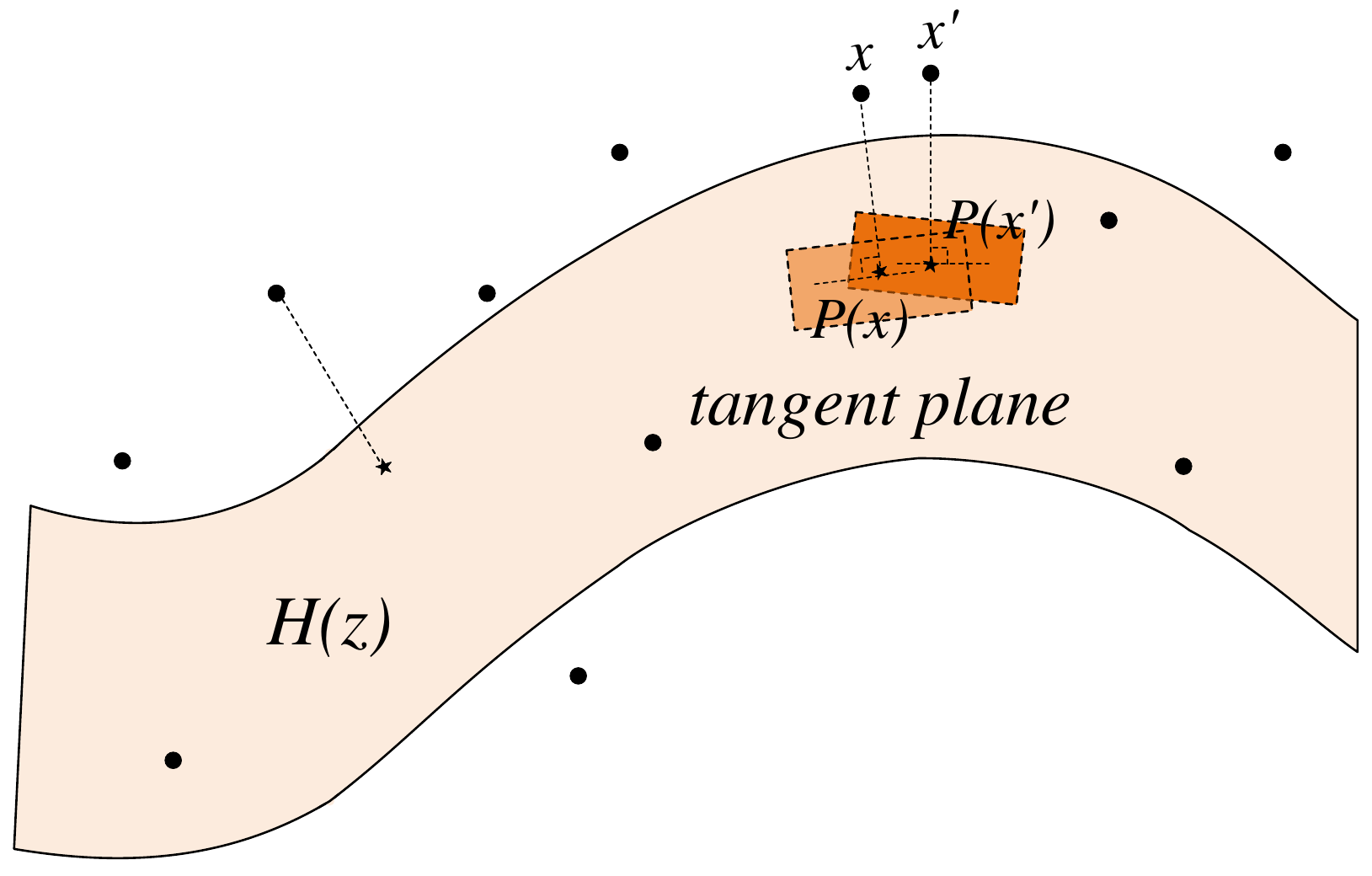}  %插入的图，包括JPG,PNG,PDF,EPS等，放在源文件目录下
	\caption{The examples of $\boldsymbol{x}$ and $\boldsymbol{x}^{\prime}$ in a $D$-dimensional space are projected onto a $d$-dimensional manifold via $P$.}  %图片的名称
	\label{fig3}   %标签，用作引用
\end{figure}

Let $U$ be an open bounded subset of $\mathbb{R}^{d}$, $\boldsymbol{z} \in U$, and $H:U(\subset \mathbb{R}^{d}) \to M$
be a local parametrization of $M$ such that
\[
H(\boldsymbol{z}) = \boldsymbol{x}_*, \quad  B(\boldsymbol{x}_*, c_0 \epsilon) \subset H(U)
\]
Since $H$ is diffeomorphism defined from $U(\subset \mathbb{R}^{d})$ to $H(U) (\subset M)$,
the set $H^{-1} \big( B(\boldsymbol{x}_*, c_0 \epsilon) \big)$ is a compact neighborhood of $\boldsymbol{z}$.

Since $H$ is an embedded submanifold, the differential ${\rm D}H(\widetilde{\boldsymbol{z}}):\mathbb{R}^{d} \to T_{H(\widetilde{\boldsymbol{z}})}M$
is injective for any point $\widetilde{\boldsymbol{z}}\in U$.
By the compactness of $H^{-1} \big( B(\boldsymbol{x}_*, c_0 \epsilon) \big)$, there are two positive constants
$c_1 < c_2$ such that
\begin{equation}
\begin{array}{c}
c_1\|\Delta \widetilde{\boldsymbol{z}}\|_2 \leq \| {\rm D}H(\widetilde{\boldsymbol{z}})[\Delta \widetilde{\boldsymbol{z}}]\|_2 \leq c_2\|\Delta \widetilde{\boldsymbol{z}}\|_2,\\
 \forall \widetilde{\boldsymbol{z}}\in H^{-1} \big( B(\boldsymbol{x}_*, c_0 \epsilon) \big), \Delta \widetilde{\boldsymbol{z}}\in \mathbb{R}^{d}
\end{array}
\end{equation}
Given a point $\boldsymbol{x}'\in B(\boldsymbol{x},\epsilon)$, denote $\boldsymbol{x}_*' = P(\boldsymbol{x}')$.
On the one hand, if $\boldsymbol {x}_* \neq \boldsymbol{x}_*'\in B(\boldsymbol{x}_*, c_0 \epsilon)$, there exists a unique point $\boldsymbol{z}'\in $ $ H^{-1} \big(B(\boldsymbol{x}_*, c_0 \epsilon)$
such that $\boldsymbol{x}_*' = H(\boldsymbol{z}')$ \cite{[a20]}. Thus we have 
\begin{equation}
\begin{aligned}
\gamma_{1}\int_{0}^1 \|  {\rm D}H\big(\boldsymbol{z}+ t(\boldsymbol{z}'-\boldsymbol{z})\big)[\boldsymbol{z}'-\boldsymbol{z}]\|_2 {\rm d}t 
\leq  {\rm dist}(\boldsymbol{x}_*',\boldsymbol{x}_*) \\ \leq \int_{0}^1 \|  {\rm D}H\big(\boldsymbol{z} + t(\boldsymbol{z}'-\boldsymbol{z})\big)[\boldsymbol{z}'-\boldsymbol{z}]\|_2 {\rm d}t
\end{aligned}
\end{equation}
where $0<\gamma_{1}\leq \frac{{\rm dist}(\boldsymbol{x}_*',\boldsymbol{x}_*)}{\int_{0}^1 \|  {\rm D}H\big(\boldsymbol{z} + t(\boldsymbol{z}'-\boldsymbol{z})\big)[\boldsymbol{z}'-\boldsymbol{z}]\|_2 {\rm d}t} $.

By (13) and (14), we have
\begin{equation}
c_1\gamma_{1}\|\boldsymbol{z}'-\boldsymbol{z}\|_2 \leq  {\rm dist}(\boldsymbol{x}_*',\boldsymbol{x}_*) \leq c_2\|\boldsymbol{z}'-\boldsymbol{z}\|_2
\end{equation}
According to (11) and (15), we can obtain
\begin{equation}
\begin{aligned}
\|\boldsymbol{z}'-\boldsymbol{z}\|_2 \leq \frac{1}{c_1\gamma_{1}} {\rm dist}(\boldsymbol{x}_*',\boldsymbol{x}_*)  &= \frac{1}{c_1\gamma_{1}}{\rm dist}(P(\boldsymbol{x}'),P(\boldsymbol{x}))\\
&\leq  \frac{c_0}{c_1\gamma_{1}}\| \boldsymbol{x}' - \boldsymbol{x}\|_2 
\end{aligned}
\end{equation}
On the other hand, if $\boldsymbol{x_*} = \boldsymbol{x_*'}$, according to the diffeomorphism of $H$, we have $\boldsymbol{z}'= \boldsymbol{z}$.\end{proof}

In summary, the above result indicates that if we obtain a seed example, it can guarantee a valid EAE in the latent space.

\section{Experiments}
\subsection{Experimental Setup}
Our experiment is implemented with PyTorch1.4, and conducted on a server with four 2080ti GPUs running Ubuntu16.04.6 LTS.

\textbf{Datasets:} Two popular benchmark datasets, CIFAR-10 and ImageNet, are used to evaluate our EAE adversarial training algorithm. CIFAR-10 has $60,000$ color images from $10$ categories with a size of $32 \times 32$. For ImageNet, we only randomly extract examples with $100$ categories and each category includes $1,350$ images with a size of $224 \times 224$.

\textbf{Models:}  (1) To select seed examples, we first train a ResNet-18 model $F_1$ on CIFAR-10 with $15$ epochs and an AlexNet model $F_3$ on ImageNet with $20$ epochs, respectively. 

(2) To generate perturbed examples for test, we train a VGG-16 model $F_2$ on CIFAR-10 with $15$ epochs and a ResNet-18 model $F_4$ on ImageNet with $20$ epochs, respectively. 

(3) On CIFAR-10, we conduct adversarial training on ResNet-18 models and on ImageNet we conduct adversarial traing on AlexNet models with EAE, FGSM, BIM, PGD, and FAST etc. Since our experiments are conducted in full black-box settings, these models are different from $F_2$ and  $F_4$.

%\begin{table}
%\caption{Some model symbols used in our experiment.}
%\vskip 0.15in
%\begin{center}
%\begin{small}
%\begin{sc}
%\scalebox{0.75}{
%\begin{tabular}{ccc}
%\toprule
%Model & Network Architecture & Fuction\\ 
%\midrule
%$F_1$    & ResNet-18 & Select seed examples\\ 
%$F_2$    & VGG-16 & Generate perturbed examples\\
%$F_3$    & AlexNet & Select seed examples\\
%$F_4$    & ResNet-18 & Generate perturbed examples\\ 
%\bottomrule
%\end{tabular}
%}
%\end{sc}
%\end{small}
%\end{center}
%\vskip -0.1in
%\label{tab0}
%\end{table}

\begin{table}
\caption{The number and MLD of seed examples and non-seed examples with respect to different perturbation bounds on CIFAR-10.}
\vskip 0.15in
\begin{center}
\begin{small}
\begin{sc}
\scalebox{0.75}{
\begin{tabular}{cccccr}
\toprule
$\epsilon$ & \#(Seed) & MLD(Seed) & \#(Non-Seed) & MLD(Non-Seed) \\ 
\midrule
$0.01$ & $2,299$ & $3.63$ & $6,883$ & $11.85$ \\
$0.02$ & $3,949$ & $4.68$ & $5,233$ & $12.36$ \\
$0.03$ & $5,051$ & $5.26$ & $4,131$ & $12.82$ \\
$0.04$ & $5,839$ & $5.66$ & $3,343$ & $13.21$ \\
$0.05$ & $6,441$ & $5.98$ & $2,741$ & $13.50$ \\ 
\bottomrule
\end{tabular}
}
\end{sc}
\end{small}
\end{center}
\vskip -0.1in
\label{tab1}
\end{table}

\begin{table}
\centering
\caption{The number and MLD of seed examples and non-seed examples with respect to different perturbation bounds on ImageNet.}
\vskip 0.15in
\begin{center}
\begin{small}
\begin{sc}
\scalebox{0.75}{
\begin{tabular}{cccccr}
\toprule
$\epsilon$     & \# (Seed) & MLD(Seed) & \# (Non-Seed) & MLD(Non-Seed) \\ \hline
$0.002$ & $241$       & $1.86$      & $1,081$          & $9.77$          \\
$0.004$ & $447$       & $3.36$      & $875$           & $10.87$         \\
$0.006$ & $642$       & $4.47$      & $680$           & $11.97$         \\
$0.008$ & $779$       & $5.17$      & $543$           & $12.87$         \\ 
\bottomrule
\end{tabular}
}
\end{sc}
\end{small}
\end{center}
\vskip -0.1in
\label{tab2}
\end{table}

\subsection{Validation of Threshold Selection}
We select seed examples on CIFAR-10 and ImageNet. As described in Sec. 4.1, an appropriate threshold plays a critical role in determining seed examples. On CIFAR-10, we perform FGSM on the model $F_1$ with perturbation bound $\epsilon=0.01$ to construct the seed example set $D^+$ and non-seed example set $D^-$, respectively. We obtain the MLD $\bar d_1=3.63$ on $D^+$ and $\bar d_2=11.85$ on $D^-$. Obviously, $\bar d_1$ is quite different to $\bar d_2$. To further explore the impact of $\epsilon$ on MLD, we continue to set $\epsilon=0.02$, $0.03$, $0.04$, and $0.05$  and repeat the above process. MLDs with respect to different $\epsilon$ are presented in Table \ref{tab1}. It shows that although the MLD of both seed examples and non-seed examples raises with the increase of $\epsilon$, the difference between them still maintains large (about $7.54\sim8.22$), which indicates the feasibility of our threshold selection mechanism. 

On ImageNet, we similarly perform FGSM on model $F_3$ with $\epsilon= 0.002$, $0.004$, $0.006$, and $0.008$, and obtain $D^+$ and $D^-$, respectively. The results are showed in Table \ref{tab2}. Similar to Table \ref{tab1}, it also demonstrates that the MLD of seed examples is smaller than that of non-seed examples by $6.76\sim7.91$.

\begin{figure} 
\centering 
\subfigure[CIFAR-10]{
% \begin{minipage}[t]{0.25\linewidth}
\centering
\includegraphics[width=1.1in]{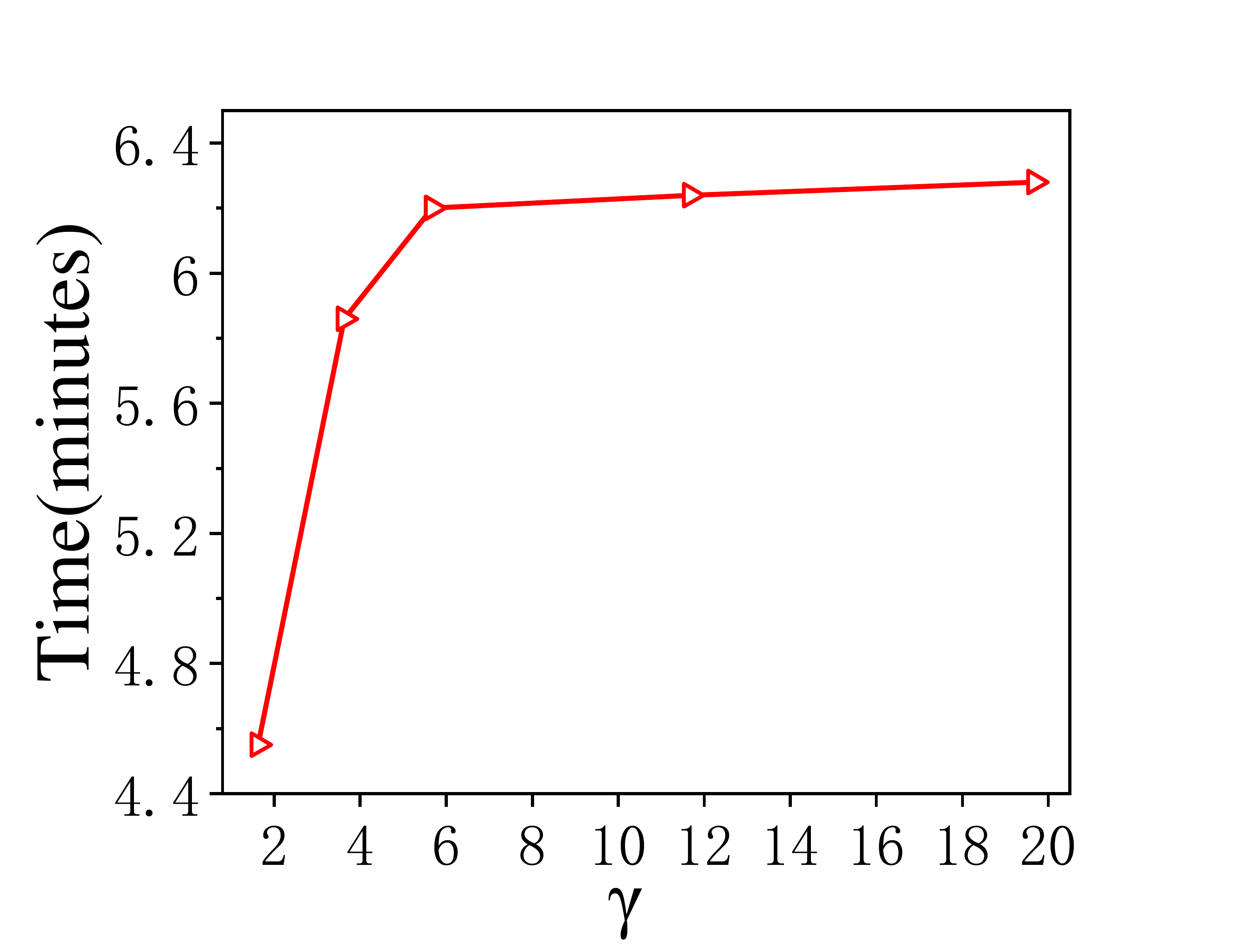}
\includegraphics[width=1.1in]{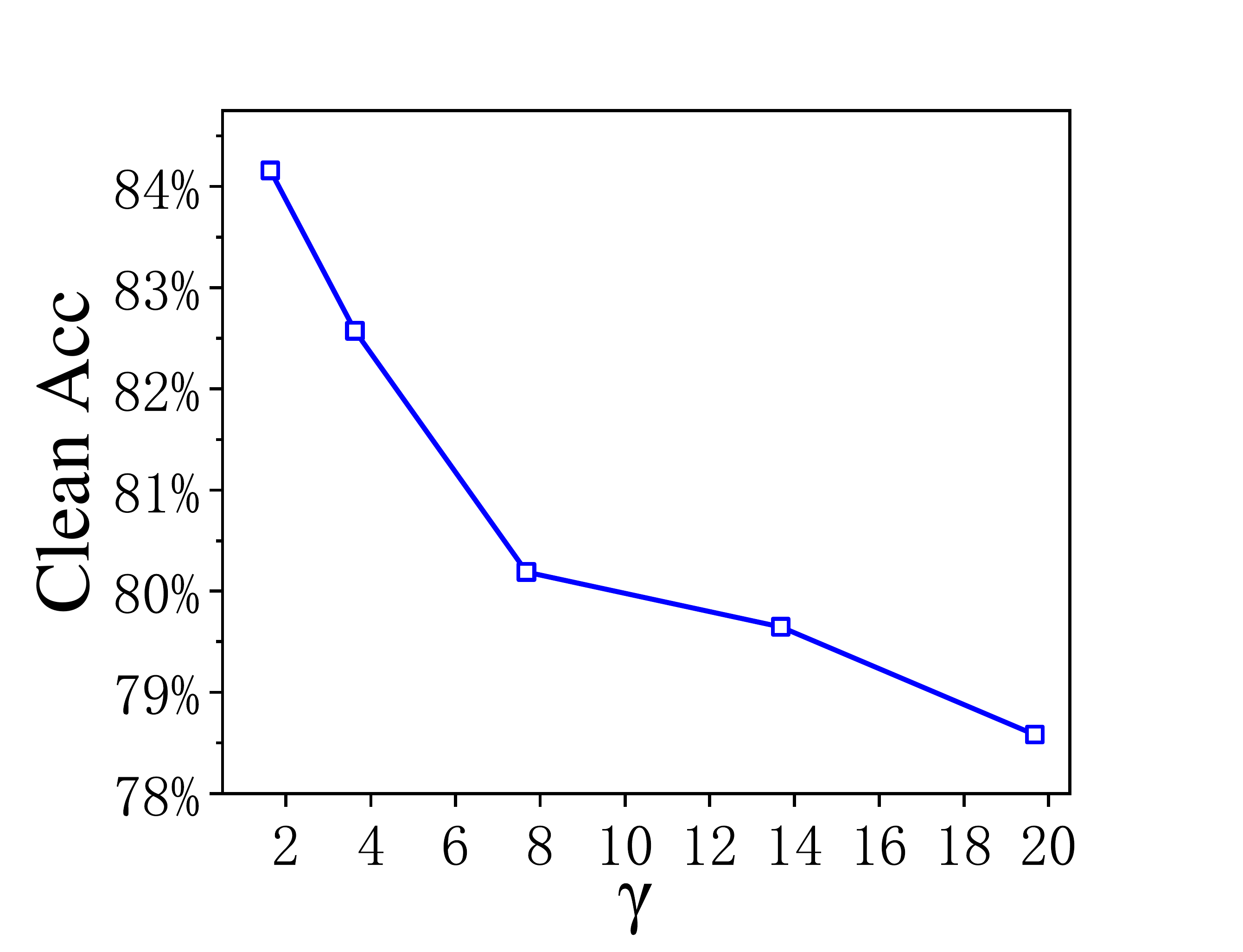}
\includegraphics[width=1.1in]{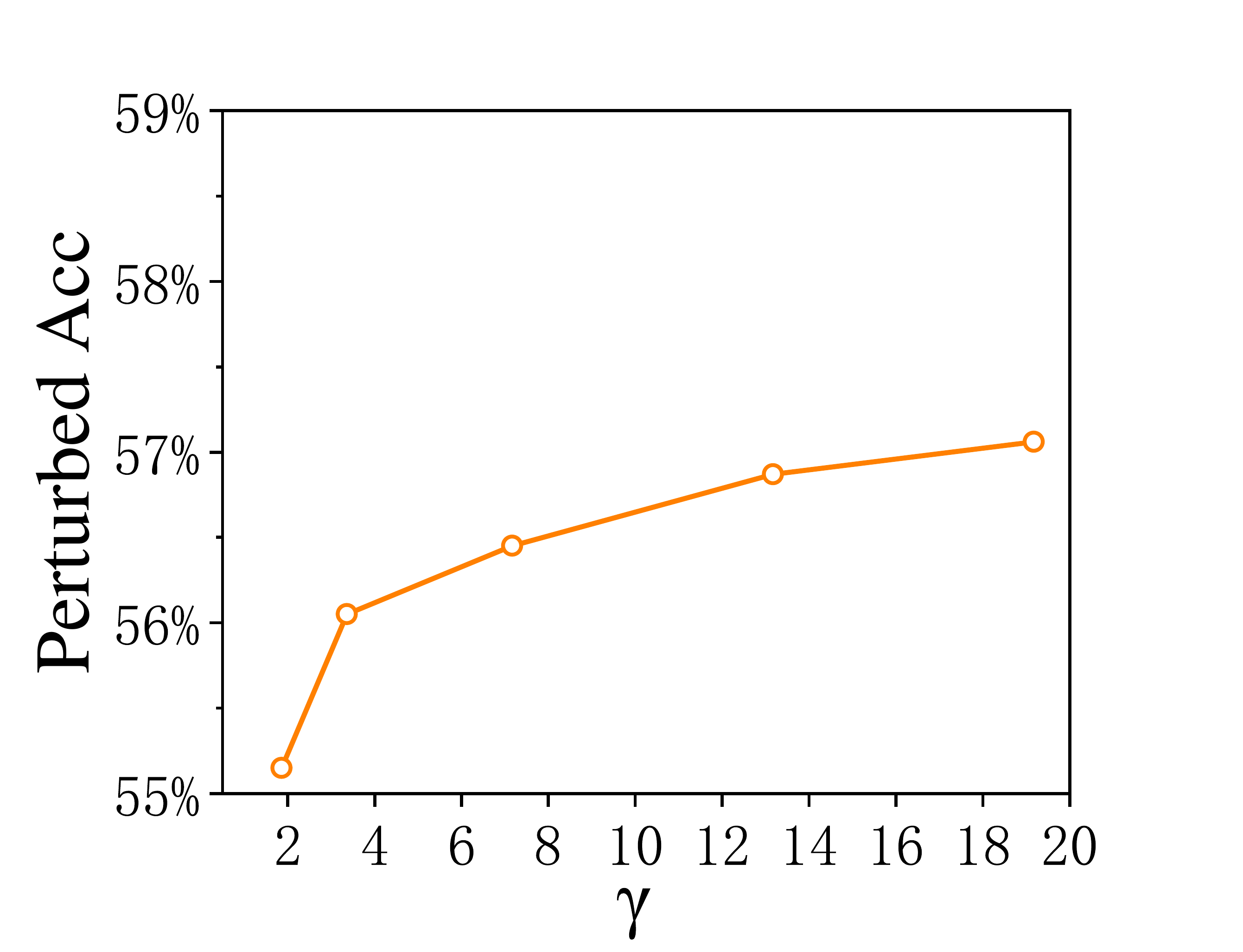} 
% \end{minipage}
}
\subfigure[ImageNet]{
\includegraphics[width=1.1in]{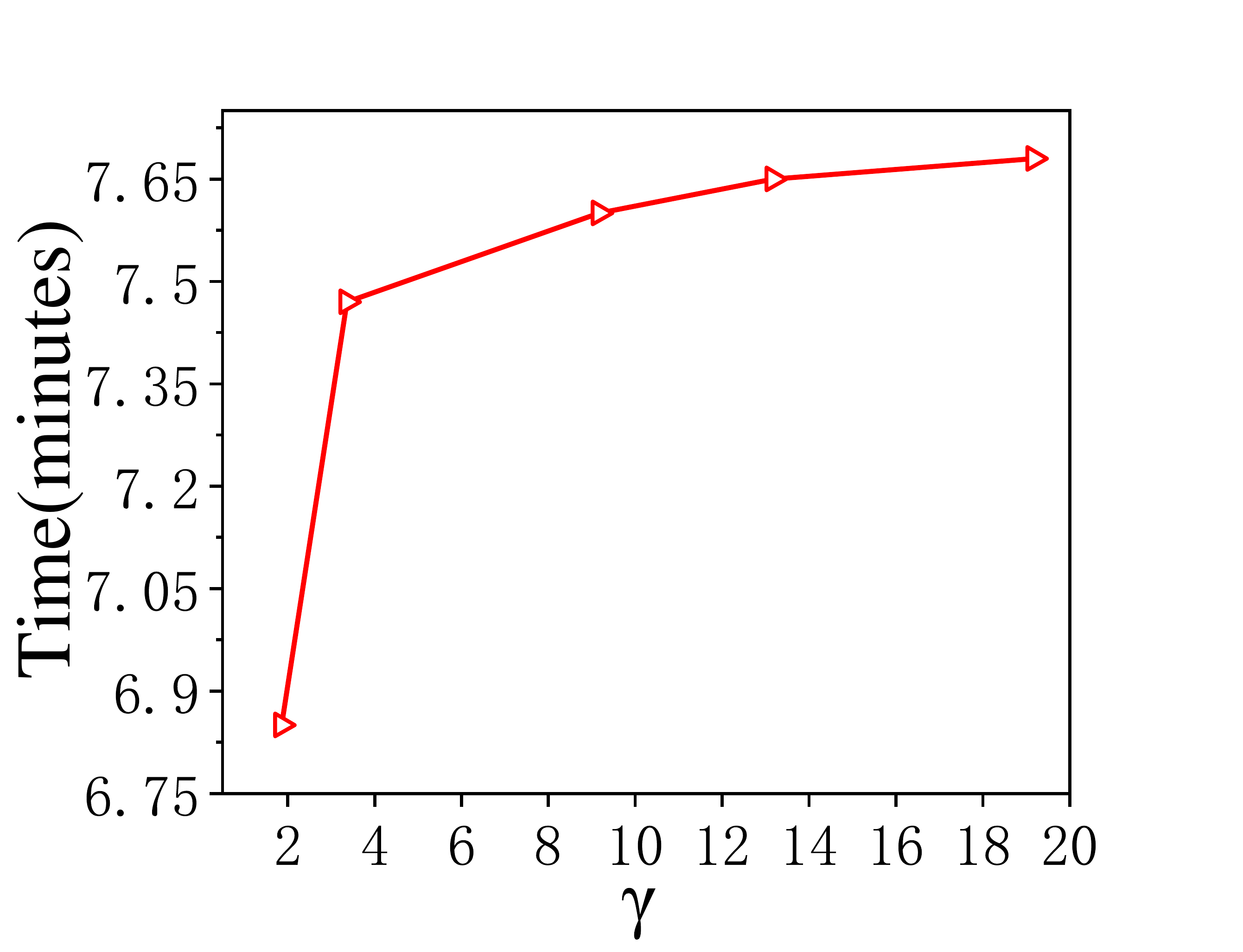}
\includegraphics[width=1.1in]{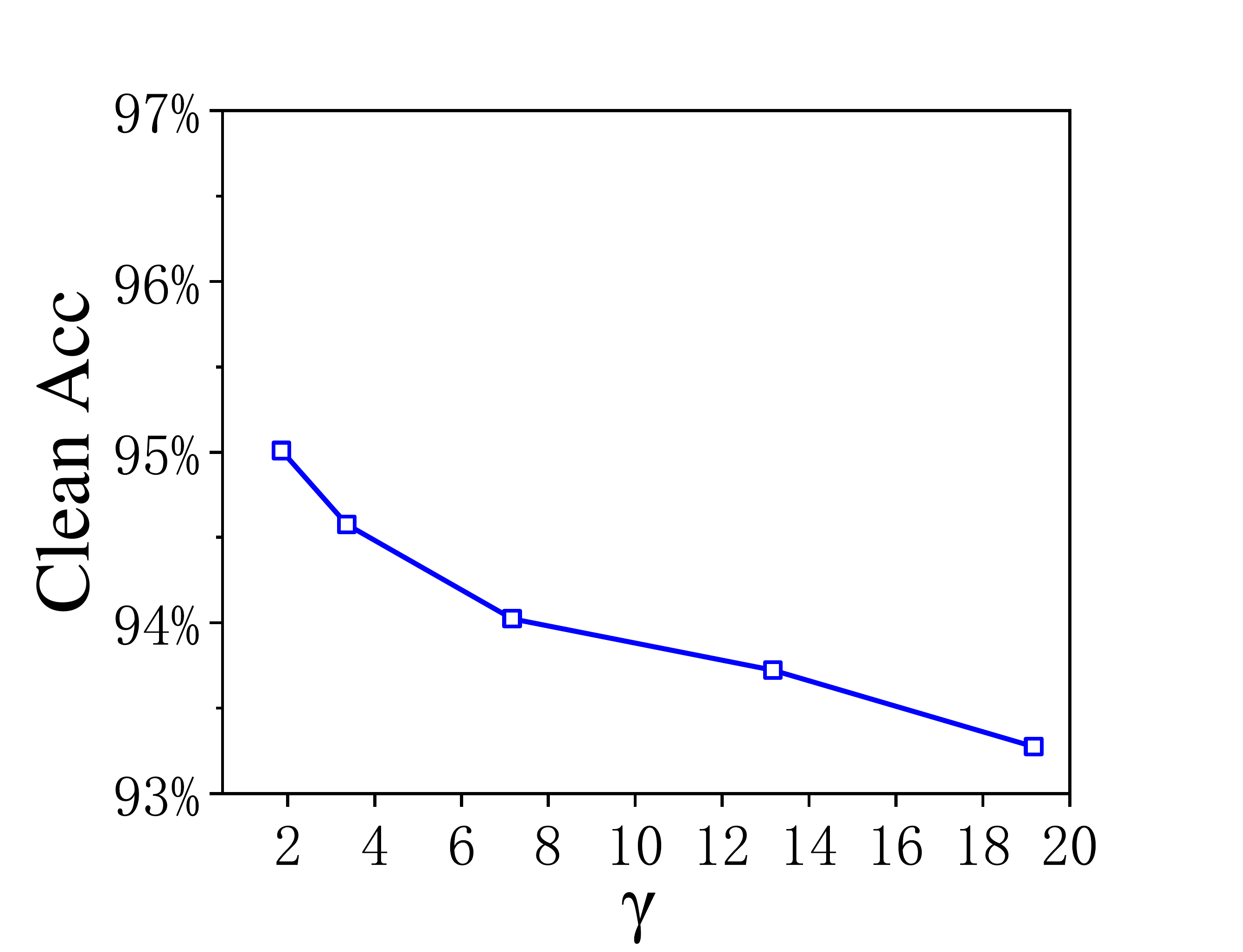}
\includegraphics[width=1.1in]{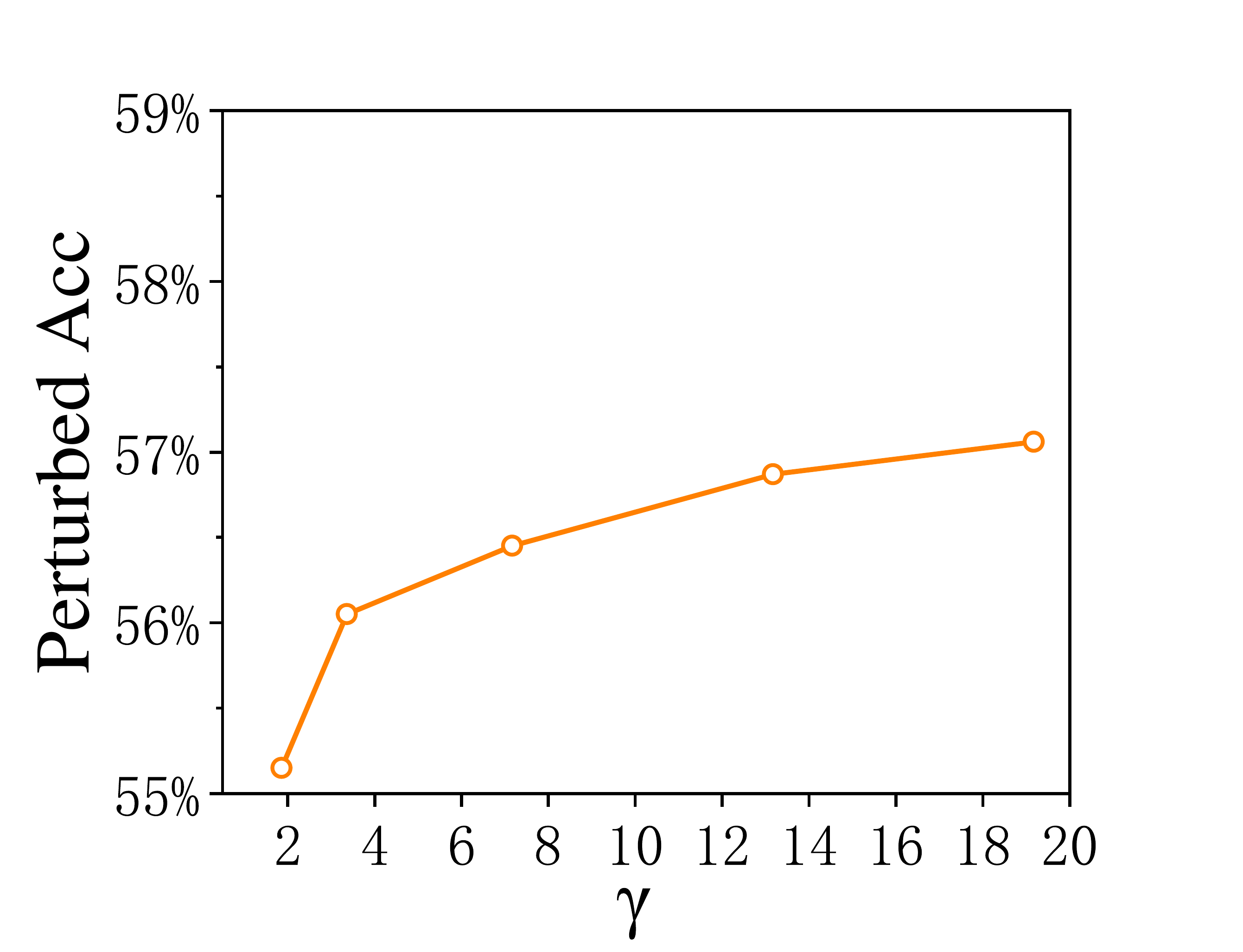} 
} 
\caption{The impact of the threshold $\gamma$ on time, clean accuracy, and perturbed accuracy of the model after EAE adversarial training: (a) the perturbed examples used for test are from CIFAR-10. They are generated by FGSM on the model $F_1$ with the perturbation bound $\epsilon=0.05$; (b) the perturbed examples used for test are from ImageNet. They are also generated by FGSM on the model $F_3$ with the perturbation bound $\epsilon=0.007$.}
\label{fig4}
\end{figure}

\subsection{Impact of Threshold}
As mentioned in Sec. 6.2, the threshold $\gamma$ is a function of the perturbation bound $\epsilon$, thus a different $\epsilon$ will lead to a different $\gamma$. Therefore, we need to determine a reasonable $\gamma$ by analyzing its influence on adversarial training with EAEs.  In this section, we demonstrate the relationship of $\gamma$ with (1) the adversarial training time, (2) the accuracy of  clean examples (\emph{Clean Acc}), and (3) the accuracy of perturbed examples (\emph{Perturbd Acc}). 

\textbf{Adversarial training time:} The first column of Fig. \ref{fig4} shows the adversarial training time with respect to the threshold, which reveals that with the raise of threshold, the adversarial training time increases accordingly. On CIFAR-10, the training time increases dramatically before $\gamma=5.98$, while after $\gamma=6$, it hardly changes. The reason is when $\gamma=6$, most of the examples have been crafted as EAEs, thus increasing the threshold will not increase the training time. Similarly, on ImageNet, when $\gamma>15$, the training time tends to be stable.

\textbf{Accuracy of clean examples:} The second column of Fig. \ref{fig4} shows the accuracy of the model after adversarial training with different thresholds on clean examples. On both CIFAR-10 and ImageNet,  the accuracy on clean examples decreases as the threshold $\gamma$ raises. Especially on CIFAR-10, the difference between the maximum and minimum prediction accuracy is about $5.9$\%, which indicates that a larger $\gamma$ will severely degrade the accuracy of clean examples than a smaller $\gamma$. 

\textbf{Accuracy of perturbed examples:} The last column of Fig. \ref{fig4} shows the accuracy of the model after adversarial training with different thresholds on perturbed examples. On CIFAR-10, the accuracy improves to some extent with the increasing threshold on perturbed examples. However, this improvement is not salient, which does not exceed $1.6$\%. On ImageNet, the accuracy of perturbed examples is only improved by $2$\% even the threshold $\gamma$ increased from $1.86$ to $19.17$, which indicates that a larger $\gamma$ does not improve the adversarial training than a smaller $\gamma$.

Based on the above analysis, we need to make a trade-off between the training time and the accuracy. A smaller threshold can speed up the adversarial training with EAEs. At the same time, it can achieve a sound accuracy on clean examples compared to a little loss of robustness. As a thumb-rule, we suggest $\gamma \in [2, 4]$ is proper for most of the data sets.

\subsection{Effectiveness of Adversarial Training}

In this section, we compare our EAE adversarial training with the following three types of adversarial training method: FGSM adversarial training, Fast adversarial training, and PGD adversarial training.

On CIFAR-10, we choose ResNet-18 as a classifier architecture. Four ResNet-18 classifiers will be obtained by four types of adversarial training with $epoch=4$, the minimum cyclic learning rate $clr_{min}=0$, and the maximum cyclic learning rate $clr_{max}=0.2$. Especially, for our EAE adversarial training, the threshold $\gamma=3.63$; for FGSM adversarial training, the perturbation bound $\epsilon=16/255$; for Fast adversarial training, $\epsilon=16/255$ and the step size $\alpha=20/255$; and for PGD adversarial training, $\epsilon=16/255$, $\alpha=8/255$, and the number of iteration $K=7$.% In addition, we also train a ResNet-18 classifier with clean examples as a benchmark on CIFAE-10. 

On ImageNet, we choose AlexNet as a classifier architecture. Four AlexNet classifiers will be obtained by four types of adversarial training with $epoch=15$, $clr_{min}=0$, and $clr_{max}=0.02$. Especially, for our EAE adversarial training, the threshold $\gamma=2.0$; for FGSM adversarial training, $\epsilon=0.01$; for Fast adversarial Training, $\epsilon=0.01$, $\alpha=3/255$; and for PGD adversarial training, $\epsilon=0.01$, $\alpha=1/255$, $K=4$. In addition, we also train two ResNet-18 classifiers with clean examples as benchmarks on CIFAR-10 and ImageNet, respectively. 

%In Fig. 5, Normal T, EAE AT, FGSM AT, Fast AT, and PGD AT represents the model after normal training (for comparison), the model after EAE adversarial training, the model after FGSM adversarial training, the model after Fast adversarial training and the model after PGD adversarial training, respectively.

 \begin{figure} 
\centering 
\subfigure[CIFAR-10]{
\includegraphics[width=3.8cm,height=3cm]{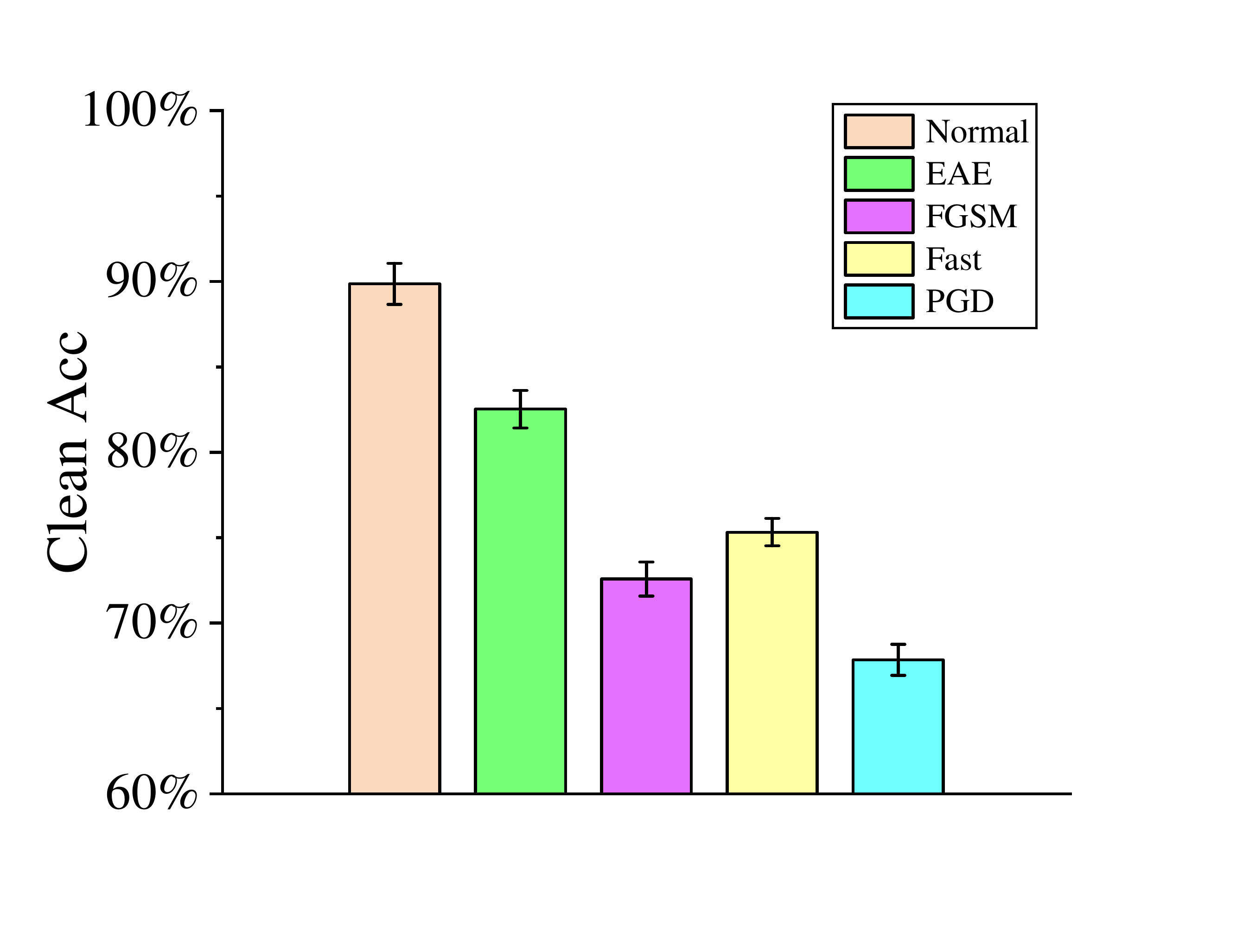} 
}
 \subfigure[ImageNet]{
\includegraphics[width=3.8cm,height=3cm]{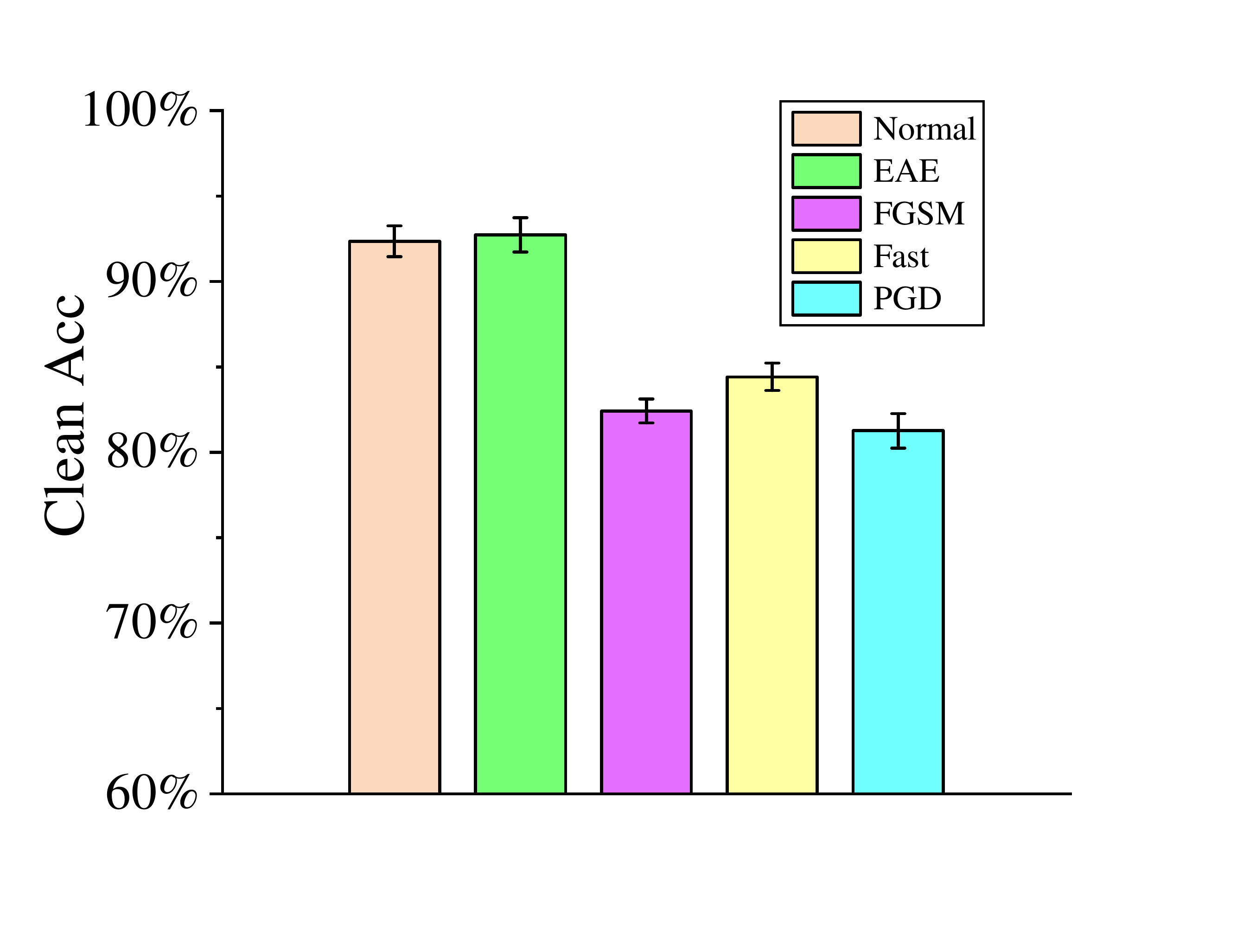} 
} 

 \caption{The accuracy of a normal trained model compared with adversarial trained models test by clean examples: (a) clean examples from CIFAR-10, all of which can be classified correctly by $F_2$; (b) clean examples from ImageNet, all of which can be classified correctly by $F_4$.}
 \label{fig5}
 \end{figure}

\textbf{The accuracy of the classifier after adversarial training tested by clean examples:} It can be seen from Fig. \ref{fig5} that compared to normal training, adversarial training in general will reduce the classification accuracy on clean examples, which has been confirmed by previous work \cite{[q]}. However, our EAE adversarial training does not decrease the accuracy of the clean examples on ImageNet, while the other three types of adversarial training decrease significantly. On CIFAR-10 the accuracy decreases as well. Nevertheless, the accuracy of our EAE adversarial training is yet higher than others by $7.21\%\sim14.68\%$.  On ImageNet, the accuracy of our EAE adversarial training is higher than others by $8.31\%\sim11.47\%$.

\begin{figure*}[htbp] 
\centering 
\subfigure[FGSM+CIFAR-10]{
\includegraphics[width=1.25in]{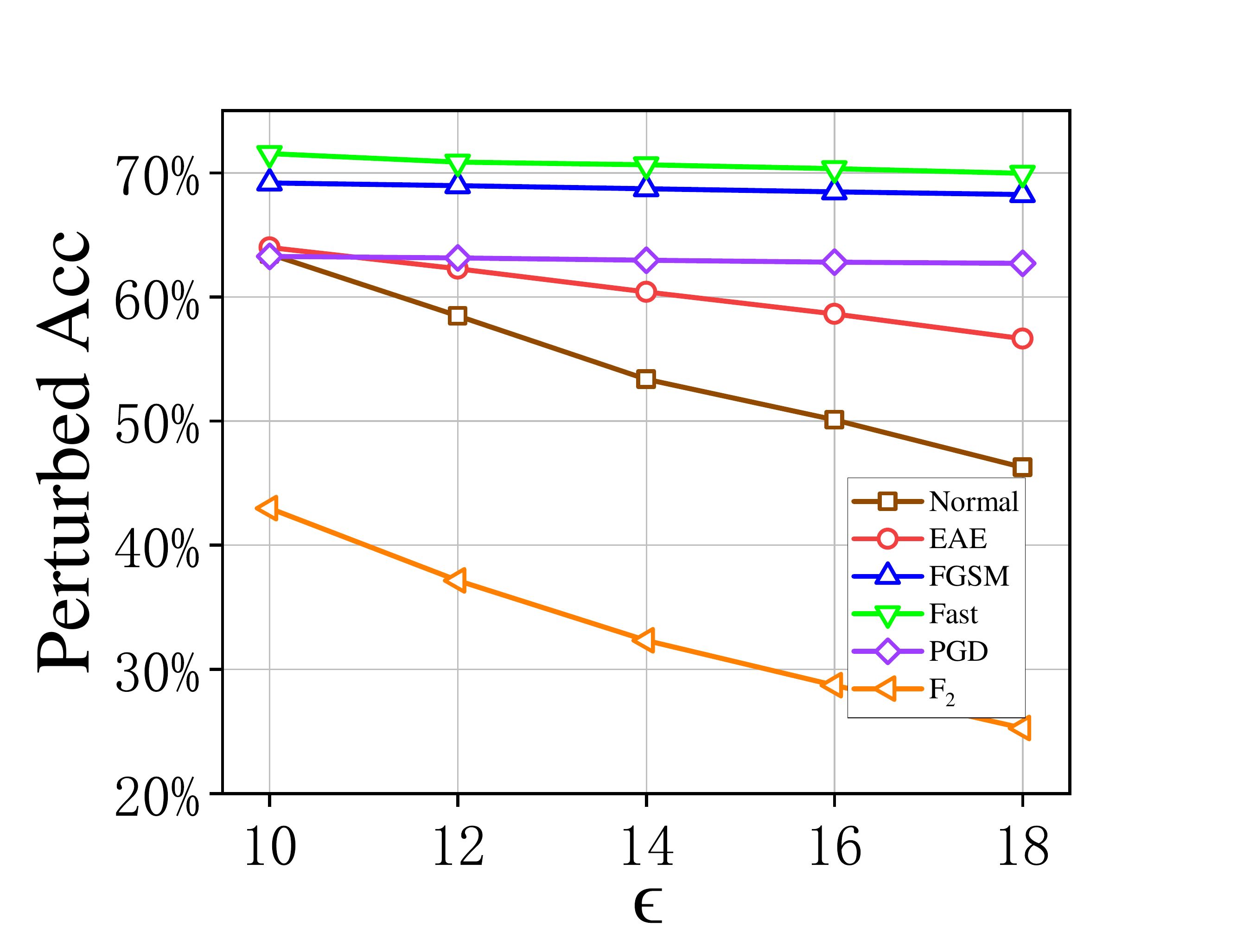} 
}
\subfigure[FGSM+ImageNet]{
\includegraphics[width=1.25in]{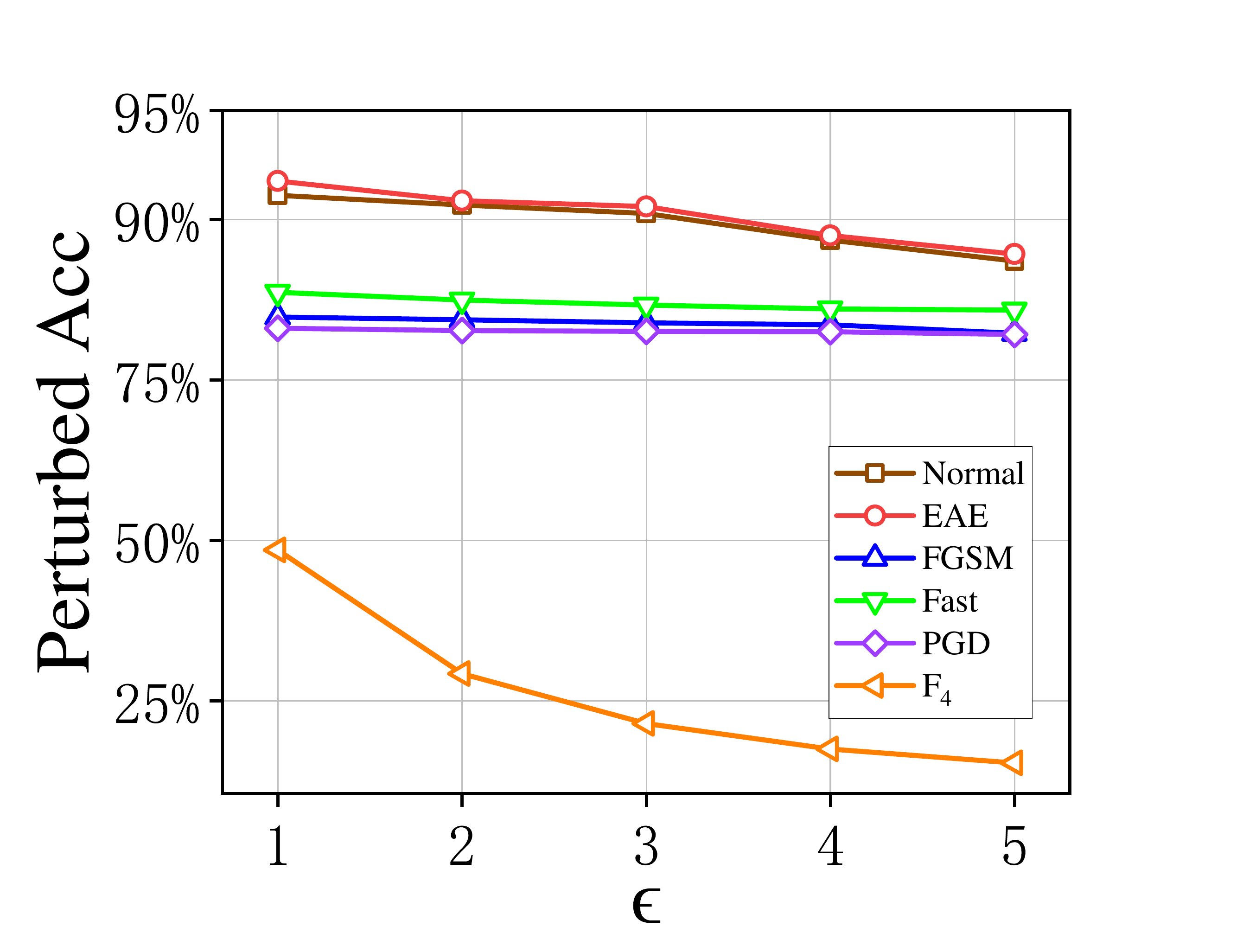} 
}
\subfigure[BIM+CIFAR-10]{
\includegraphics[width=1.25in]{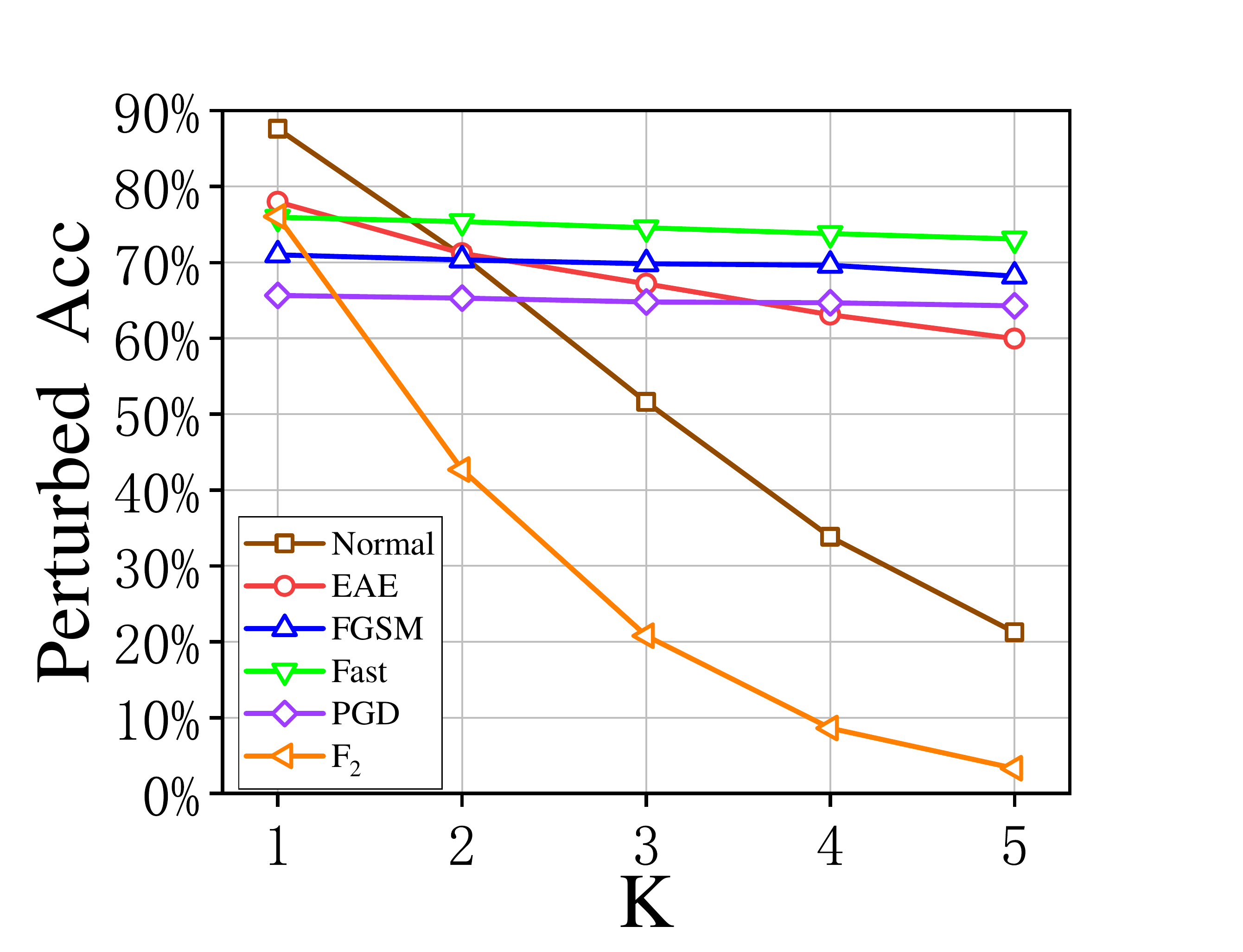} }
\subfigure[BIM+ImageNet]{
\includegraphics[width=1.25in]{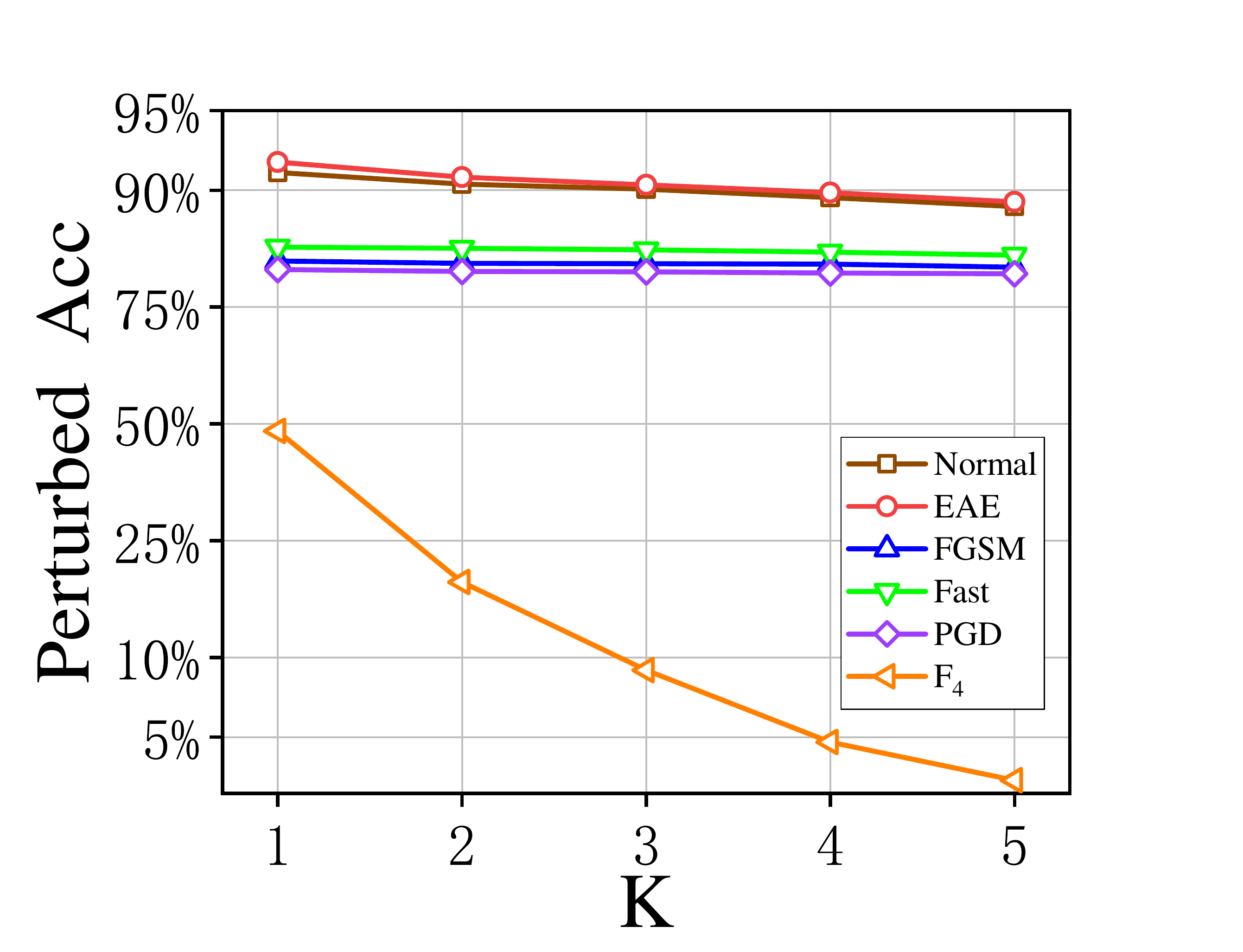} 
} 
\subfigure[PGD+CIFAR-10]{
\includegraphics[width=1.25in]{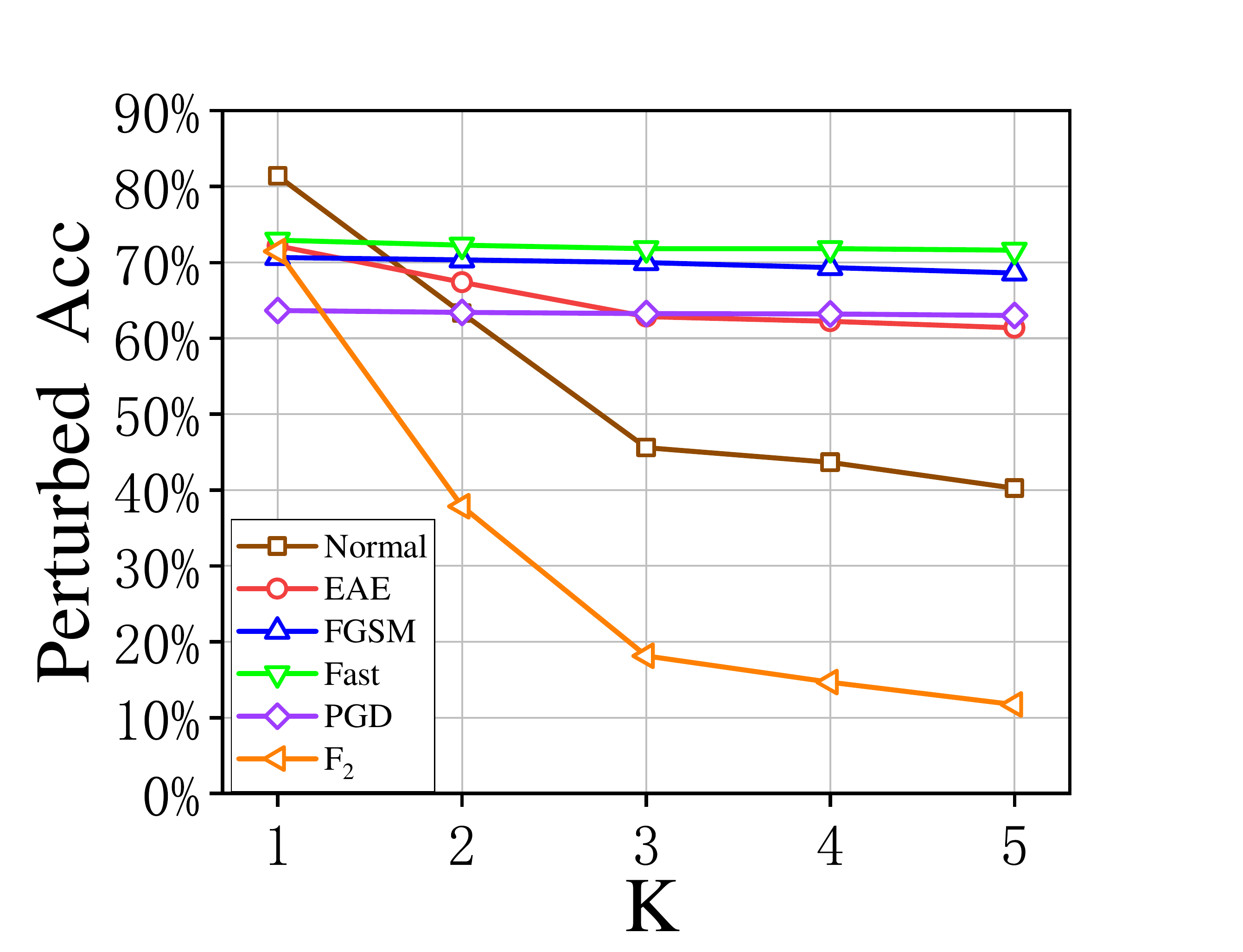} 
}
\subfigure[PGD+ImageNet]{
\includegraphics[width=1.25in]{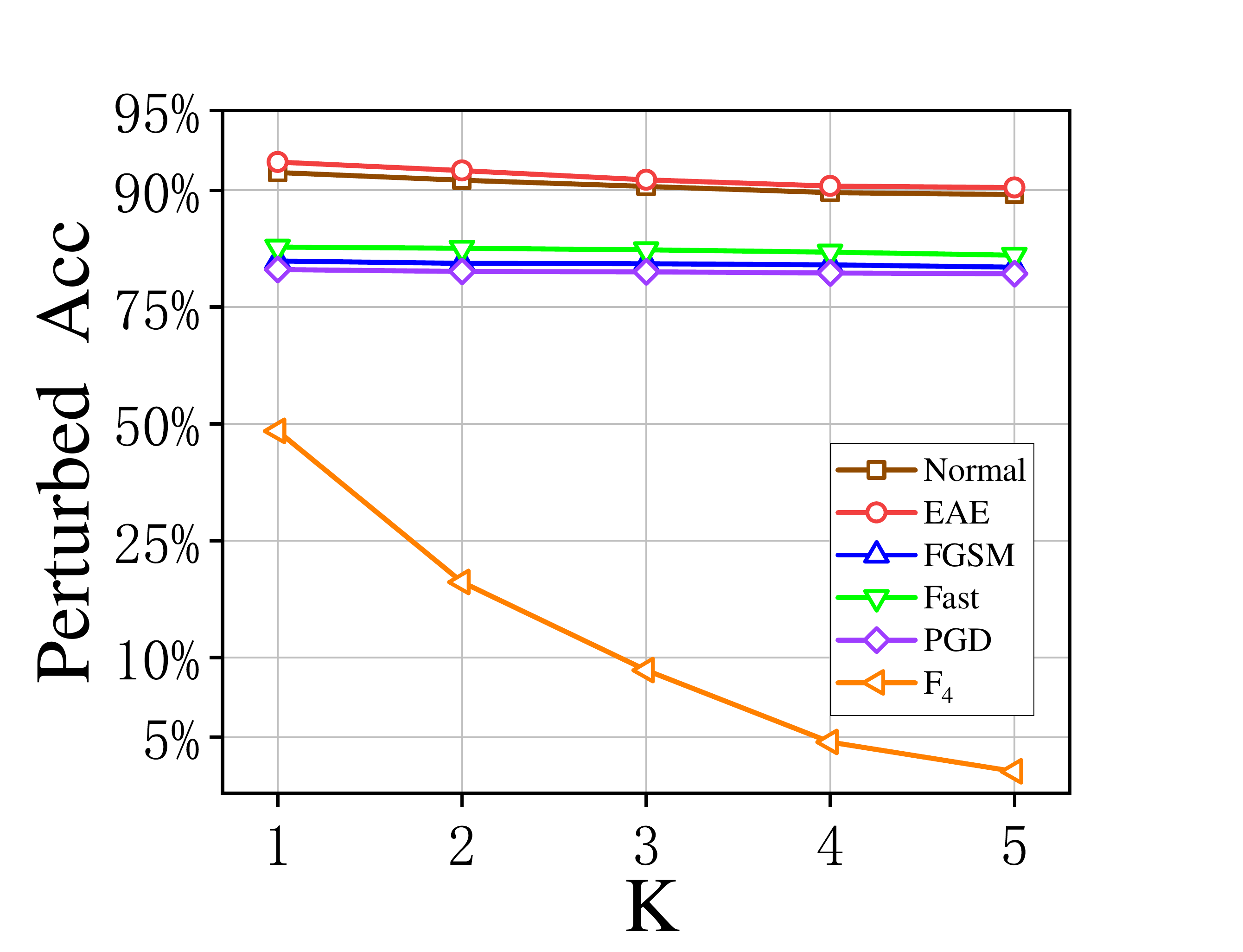} 
} 
\subfigure[RFGSM+CIFAR-10]{
\includegraphics[width=1.25in]{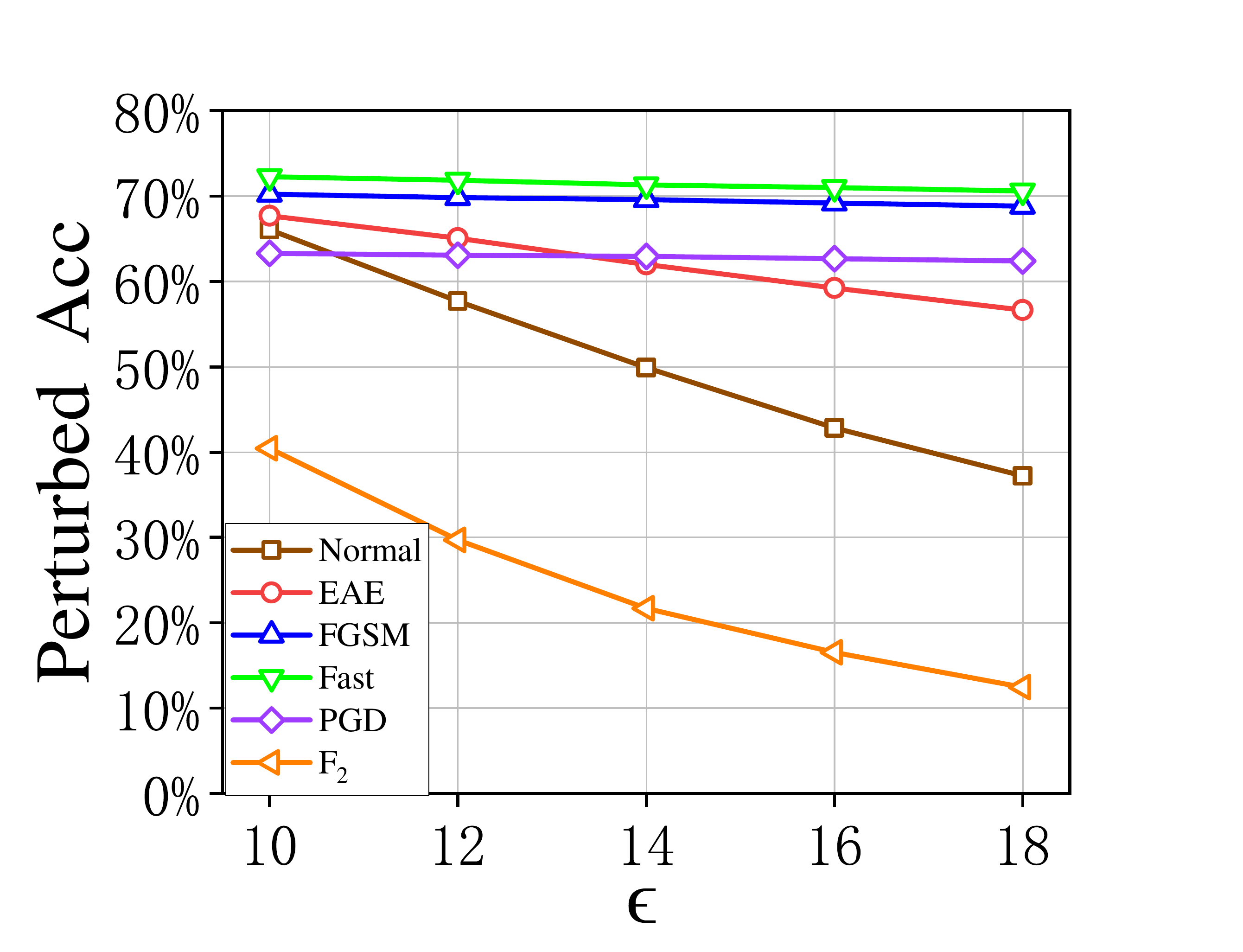} 
}
\subfigure[RFGSM+ImageNet]{
\includegraphics[width=1.25in]{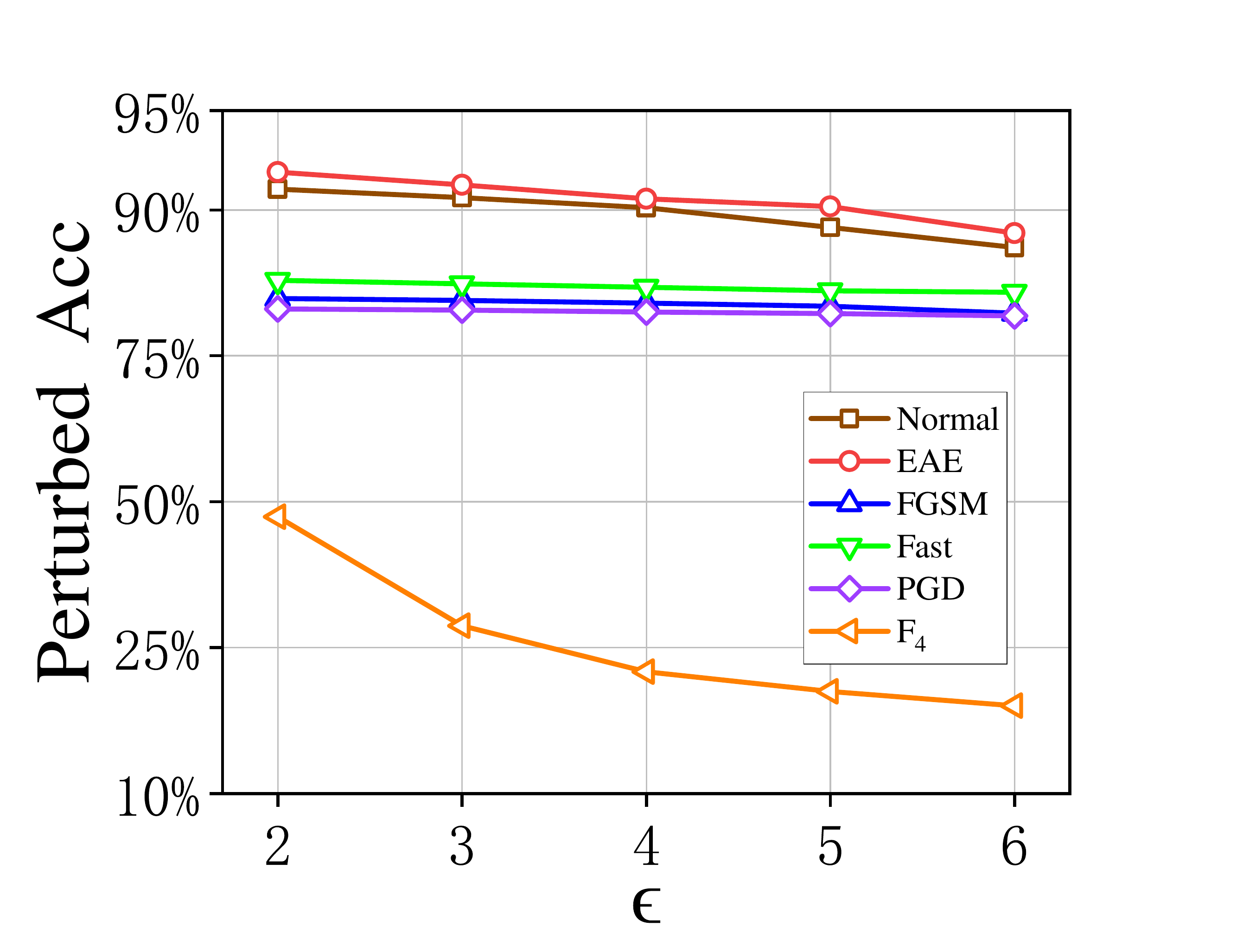} 
} 
\subfigure[Fast+CIFAR-10]{
\includegraphics[width=1.25in]{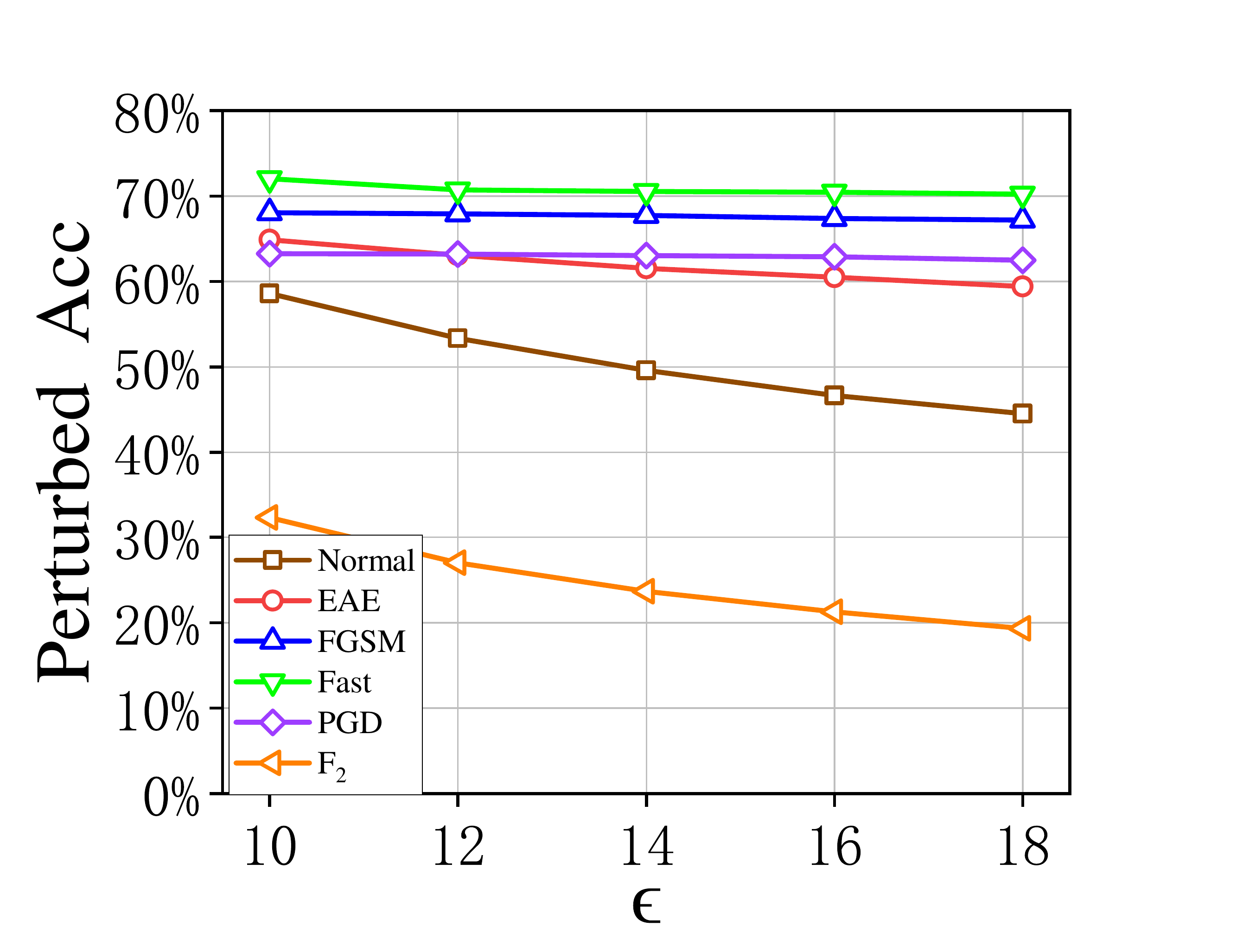} 
}
\subfigure[Fast+ImageNet]{
\includegraphics[width=1.25in]{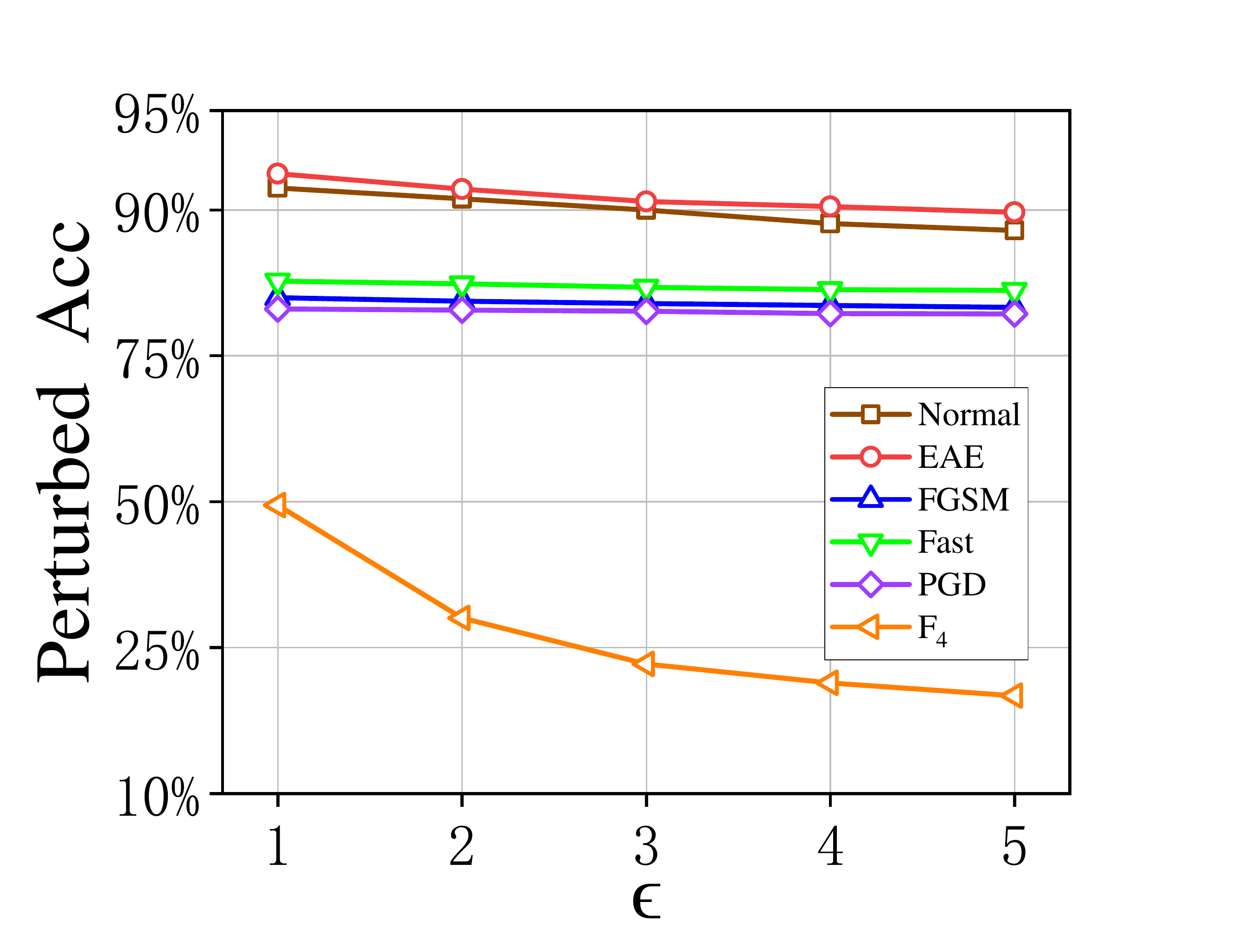} 
} 
\caption{The accuracy of the models with Normal training, EAE, FGSM, Fast, and PGD adversarial training tested by the perturbed examples. (a)$\sim$ (j) present ``attack+dataset" to generate perturbed examples for test with paremeters: (a) $\epsilon/255$; (b) $\epsilon/255$; (c) $\epsilon=12/255$, $\alpha=4/255$; (d) $\epsilon=2/255$, $\alpha/255$; (e) $\epsilon=12/255$, $\alpha=4/255$; (f) $\epsilon=4/255$, $\alpha=1/255$; (g) $\epsilon/255$, $\alpha=4/255$, $K=1$; (h) $\epsilon=1/255$, $\alpha/255$, and $K=1$; (i) $\epsilon/255$, $\alpha=20/255$; (j) $\epsilon/255$, $\alpha=6/255$. Note on CIFAR-10, the model to generate perturbed examples is $F_2$ (VGG-16), while on ImageNet, it is $F_4$ (ResNet-18).
\label{fig7}
}
\end{figure*}

Why does our EAE adversarial training not reduce the accuracy of clean examples? Since our method generating EAEs is to move the example to the decision boundary, and they will still be correctly classified with a probability of $50\%$. Thus, with our method the parameters are fine-tuned and the decision boundary merely changes less than other methods, as shown in Fig. \ref{fig14}. Hence, it has tiny influence on the accuracy of the clean examples. Remind other adversarial training methods, e.g., PGD,  the adversarial examples that they employ may cross and go far from the decision boundary, which leads the adversarial training to adjust the decision boundary with a larger magnitude. As a result, some clean examples are misclassified by their newly adjusted decision boundary.

\textbf{The accuracy of the classifier after adversarial training tested by perturbed examples:} Fig. \ref{fig7} presents the accuracy of the models with normal training or adversarial training with respect to the perturbation bound $\epsilon$. Remind that $F_2$ and $F_4$ are used to generate perturbed examples respectively on CIFAR-10 and ImageNet  to evaluate the robustness of the classifier, respectively. We observed that the proportion of adversarial examples increases with the perturbation bound $\epsilon$ or the number of iterations, which is consistent with our intuition. In essence, the accuracy of $F_2$ or $F_4$ represents the proportion of adversarial examples in the perturbed examples, and the lower accuracy of $F_2$ or $F_4$ indicates the more adversarial examples contained. 

We observed that on CIFAR-10, increasing $\epsilon$ makes the accuracy decrease by $2\% \sim 15\%$ for our EAE adversarial training compared to the other three adversarial training methods that are not obviously affected by $\epsilon$. However, compared with the normal training, our EAE adversarial training can improve the accuracy by $11\%\sim39\%$, which indicates it can enhance the robustness of the model. 
%At the same time, it also shows that clean examples are more inclined to generate adversarial examples to the second class. Because in our EAE adversarial training, only the maximum value and the second largest value in logits are adjusted, which correspond to the first and second classes of the final output. 

On ImageNet, the robustness of the model after EAE adversarial training is better than other adversarial training methods, and its accuracy is higher than other methods by about $7\%\sim10\%$. We also found that the robustness of the model after FGSM, Fast, and PGD adversarial training is even worse than that of normal training. Especially, PGD adversarial training is the worst among them. %A possible reason is that the number of iterations is not enough in our experiment.

Why can our EAE adversarial training effectively improve the robustness of the model? As shown in Fig. \ref{fig14}, given a perturbation bound, if a clean example intends to be crafted as an adversarial example, our method moves the clean example towards the closest decision boundary. In fact, we merely change the maximum and the second largest component in the logits during the adversarial training, and then update the parameters through back-propagation. The decision boundary adjusted by our method can defend against a large number of adversarial examples whose corresponding clean examples belong to Class1 can easily be crafted into adversarial examples misclassified as Class2.

\begin{table}[htbp]
\caption{The training time required for four types of adversarial trainings}
% \vskip 0.15in
\begin{center}
\begin{small}
\begin{sc}
\scalebox{0.85}{
\begin{tabular}{lccccr}
\toprule
Dataset  & EAE   & FGSM  & Fast  & PGD    \\ 
\midrule
CIFAR-10   & $2.86$ min & $7.63$ min & $7.73$ min & $17.37$ min\\
ImageNet & $4.25$ hrs & $12.54$ hrs & $15.59$ hrs & $53.45$ hrs\\ 
\bottomrule
\end{tabular}
}
\end{sc}
\end{small}
\end{center}
% \vskip -0.1in
\label{tab3}
\end{table}

\subsection{Time Cost of Adversarial Training}
In Table \ref{tab3}, we can clearly see that our EAE adversarial training is faster than other methods, since it does not need to generate real adversarial examples, and hence does not need to calculate the gradient of the loss with respect to the example. Compared to state-of-the-art ``Fast" adversarial training on CIFAR-10 and ImageNet, our method is faster by $4.87$ minutes and $11.34$ hrs, i.e., reducing the time by $63\%$ and $72.74\%$, respectively.

\section{Conclusion}
In this paper, we propose a novel EAE adversarial training method without generating real adversarial examples to participate in training. The experiment results show that our method can speed up the adversarial training process, which outperforms state-of-the-art ``Fast'' method. Especially,our EAE adversarial training has little impact on the accuracy of clean examples, which is a challenge in previous methods. In the future, we will continue to explore other techniques, e.g., network pruning, to speed up adversarial training for a robust network. Besides this, it is likely to further improve the EAE-based adversarial training itself.

% \section*{Acknowledgements}

% \textbf{Do not} include acknowledgements in the initial version of
% the paper submitted for blind review.

% If a paper is accepted, the final camera-ready version can (and
% probably should) include acknowledgements. In this case, please
% place such acknowledgements in an unnumbered section at the
% end of the paper. Typically, this will include thanks to reviewers
% who gave useful comments, to colleagues who contributed to the ideas,
% and to funding agencies and corporate sponsors that provided financial

\nocite{[a21]}

\bibliography{example_paper}
\bibliographystyle{icml2021}

% \appendix
% \section{Do \emph{not} have an appendix here}

% \textbf{\emph{Do not put content after the references.}}
% %
% Put anything that you might normally include after the references in a separate
% supplementary file.

% We recommend that you build supplementary material in a separate document.
% If you must create one PDF and cut it up, please be careful to use a tool that
% doesn't alter the margins, and that doesn't aggressively rewrite the PDF file.
% pdftk usually works fine. 

% \textbf{Please do not use Apple's preview to cut off supplementary material.} In
% previous years it has altered margins, and created headaches at the camera-ready
% stage. 
%%%%%%%%%%%%%%%%%%%%%%%%%%%%%%%%%%%%%%%%%%%%%%%%%%%%%%%%%%%%%%%%%%%%%%%%%%%%%%%
%%%%%%%%%%%%%%%%%%%%%%%%%%%%%%%%%%%%%%%%%%%%%%%%%%%%%%%%%%%%%%%%%%%%%%%%%%%%%%%

\end{document}